%% file: main.tex
\documentclass[10pt,twocolumn,letterpaper]{article}

\usepackage[pagenumbers]{cvpr} 

\usepackage{graphicx}
\usepackage{amsmath}
\usepackage{amssymb}
\usepackage{amsthm}
\usepackage{cite}
\usepackage{algorithm}
\usepackage{algpseudocode}
\usepackage{booktabs}
\usepackage{amsfonts}       
\usepackage{multirow}
\usepackage{tabularx}

\makeatletter
\newcommand{\multiline}[1]{%
    \begin{tabularx}{\dimexpr\linewidth-\ALG@thistlm}[t]{@{}X@{}}
        #1
    \end{tabularx}
}

\newtheorem{theorem}{Theorem}
\newtheorem{lemma}[theorem]{Lemma}
\newtheorem*{theorem*}{Theorem}

\usepackage[dvipsnames]{xcolor}
\makeatletter
\@namedef{ver@everyshi.sty}{}
\makeatother
\usepackage{tikz}
\usetikzlibrary{backgrounds}
\usetikzlibrary{arrows,shapes}
\usetikzlibrary{tikzmark}
\usetikzlibrary{calc}

\usepackage{mathtools, nccmath}
\usepackage{wrapfig}
\usepackage{comment}

\usepackage{blindtext}

\usepackage{xspace}

\usepackage{array}
\usepackage{ragged2e}
\newcolumntype{P}[1]{>{\RaggedRight\hspace{0pt}}p{#1}}
\newcolumntype{X}[1]{>{\RaggedRight\hspace*{0pt}}p{#1}}

\usepackage{tcolorbox}
\usepackage{paralist,tabularx}

\usepackage{tikz}
\usetikzlibrary{arrows,shapes,positioning,shadows,trees,mindmap}
\usepackage[edges]{forest}
\usetikzlibrary{arrows.meta}
\colorlet{linecol}{black!75}
\usepackage{xkcdcolors} 
\usepackage{tikz}
\usetikzlibrary{backgrounds}
\usetikzlibrary{arrows,shapes}
\usetikzlibrary{tikzmark}
\usetikzlibrary{calc}
\newcommand{\highlight}[2]{\colorbox{#1!17}{$\displaystyle #2$}}

\colorlet{mhpurple}{Plum!80}
\renewcommand{\highlight}[2]{\colorbox{#1!17}{#2}}

\def\*#1{\mathbf{#1}}
\newcommand{\bx}{\mathbf{x}}

\newcommand{\methodname}{SA-FAS }

\usepackage[pagebackref,breaklinks,colorlinks,citecolor=blue]{hyperref}

\usepackage[capitalize]{cleveref}
\crefname{section}{Sec.}{Secs.}
\Crefname{section}{Section}{Sections}
\Crefname{table}{Table}{Tables}
\crefname{table}{Tab.}{Tabs.}

\def\cvprPaperID{3950} 
\def\confName{CVPR}
\def\confYear{2023}

\begin{document}

\title{Rethinking Domain Generalization for Face Anti-spoofing: \\ Separability and Alignment}

\author{\\
Yiyou Sun$^{2}$\thanks{This work was done during Yiyou Sun’s internship at Google.}, Yaojie Liu$^1$, Xiaoming Liu$^{1,3}$, Yixuan Li$^2$, Wen-Sheng Chu$^1$\\
$^1$Google Research, $^2$University of Wisconsin-Madison, $^3$Michigan State University\\
{\tt\small $^1$\{yaojieliu,xiaomingl,wschu\}@google.com, $^2$\{sunyiyou,sharonli\}@cs.wisc.edu, $^3$liuxm@cse.msu.edu}
}
\maketitle

\def\eqnvspace{{\vspace{-2mm}}}
\def\tabvspace{{\vspace{-1mm}}}
\def\figvspace{{\vspace{-3mm}}}
\newcommand{\norm}[1]{\left\lVert#1\right\rVert}

\newcommand{\Paragraph}[1]{\vspace{1mm} \noindent \textbf{#1} \hspace{0mm}}
\newcommand{\Section}[1]{\vspace{-2mm} \section{#1} \vspace{-1mm}}
\newcommand{\SubSection}[1]{\vspace{-1mm} \subsection{#1} \vspace{-1mm}}
\newcommand{\SubSubSection}[1]{\vspace{-1mm} \subsubsection{#1} \vspace{-1mm}}

\newlength\segsep
\setlength{\segsep}{-0.5ex}

\begin{abstract}
This work studies the generalization issue of face anti-spoofing (FAS) models on domain gaps, such as image resolution, blurriness and sensor variations. Most prior works regard domain-specific signals as a negative impact, and apply metric learning or adversarial losses to remove them from feature representation. 
Though learning a domain-invariant feature space is viable for the training data, we show that the feature shift still exists in an unseen test domain, which backfires on the generalizability of the classifier. 
In this work, instead of constructing a domain-invariant feature space, we encourage domain separability while aligning the live-to-spoof transition (i.e., the trajectory from live to spoof) to be the same for all domains. 
We formulate this FAS strategy of separability and alignment (SA-FAS) as a problem of invariant risk minimization (IRM), and learn domain-variant feature representation but domain-invariant classifier. We demonstrate the effectiveness of SA-FAS on challenging cross-domain FAS datasets and establish state-of-the-art performance. Code is available at \small{\url{https://github.com/sunyiyou/SAFAS}}.
\end{abstract}


\input{sec1-intro}
\input{sec2-prior}
\input{sec3-method}
\input{sec4-exp}

\input{sec5-conclusion}


{\small
\bibliographystyle{ieee_fullname}
\bibliography{egbib}
}


\input{appendix}

\end{document}

%% file: sec1-intro.tex
\Section{Introduction}
\label{sec:intro}
\vspace{\segsep}

Face recognition (FR)~\cite{deng2019arcface} has achieved remarkable success and has been widely employed in mobile access control and electronic payments. Despite the promise, FR systems still suffer from presentation attacks (PAs), including print attacks, digital replay, and 3D masks. As a result, face anti-spoofing (FAS) has been an important topic for almost two decades~\cite{yang2014learn,liu2019deep,wang2022patchnet,atoum2017face,liu2018learning,yu2020face, kim2019basn}. 

\input{subtex/figure_2}

In early systems like building access and border control with limited variations (\eg, lighting and poses), simple methods~\cite{boulkenafet2015face,freitas2012lbp,li2016original} have exhibited promise. These algorithms are designed for the closed-world setting, where the camera and environment are assumed to be the same between train and test. This assumption, however, rarely holds for in-the-wild applications, \eg, mobile face unlock and sensor-invariant ID verification.
Face images in those FAS cases may be acquired from wider angles, complex scenes, and different devices, where it is hard for training data to cover all the variations.
These differences between training and test data are termed domain gaps  and the FAS solutions to tackle the domain gaps are termed cross-domain FAS.

Learning domain-invariant representation is the main approach in generic domain generalization~\cite{wang2022generalizing}, and has soon been widely applied to cross-domain FAS~\cite{wang2019improving,shao2019multi,jia2020ssdg,liu2021dual,liu2021adaptive,wang2022ssan}.
Those methods consider domain-specific signals as a confounding factor for model generalization, and hence aim to remove domain discrepancy from the feature representation partially or entirely. 
Adversarial training is commonly applied so that upon convergence the domain discriminator cannot distinguish which domain the features come from.
In addition, some methods apply metric learning to further regularize the feature space, \eg,~triplet loss~\cite{wang2019improving}, dual-force triplet loss~\cite{shao2019multi}, and single-side triplet loss~\cite{jia2020ssdg}.

%
There are two crucial issues that limit the generalization ability of these methods~\cite{wang2019improving,shao2019multi,jia2020ssdg,liu2021dual,liu2021adaptive,wang2022ssan} with domain-invariant feature losses. First, 
these methods posit a strong assumption that the feature space is perfectly domain-invariant after removing the domain-specific signals from training data. 
However, this assumption is unrealistic due to the limited size and domain variants of the training data, on which the loss might easily overfit during training. As shown in Fig.~\ref{fig:cmp_da_umap}, the test distribution is more expanded compared to the training one, and the spatial relation between live and spoof has largely deviated from the learned classifier.
Second, feature space becomes ambiguous when domains are mixed together. Note that the domain can carry information on certain image resolutions, blurriness and sensor patterns.
If features from different domains are collapsed together~\cite{papyan2020prevalence}, the live/spoof classifier will undesirably leverage spurious correlations to make the live/spoof predictions as shown in Fig.~\ref{fig:teaser} (a), \eg, comparing live from low-resolution domains to spoof from high-resolution ones. Such a classifier will unlikely generalize to a test domain when the correlation does not exist.

%
 
In this work, we rethink feature learning for cross-domain FAS. Instead of constructing a domain-invariant feature space, we aim to find a generalized classifier while explicitly maintaining domain-specific signals in the representation. Our strategy can be summarized by the following two properties:
\begin{itemize}
    \tabvspace\item \textbf{Separability:} We encourage features from different domains and live/spoof classes to be separated which facilitates maintaining the domain signal.
    According to~\cite{ben2006analysis}, representations with well-disentangled domain variation and task-relevant features are more general and transferable to different domains. 
    \tabvspace\item \textbf{Alignment:} Inspired by~\cite{jourabloo2018face}, we regard spoofing as the process of transition. For similar PA types\footnote{This work focuses on print and replay attacks.}, the transition process would be similar, regardless of environments and sensor variations. With this assumption, we regularize the live-to-spoof transition to be aligned in the same direction for all domains. 
\end{itemize}
We refer to this new learning framework as \textit{FAS with separability and alignment} (dubbed \textbf{SA-FAS}), shown in Fig.~\ref{fig:teaser} (b).
To tackle the separability, we leverage Supervised Contrastive Learning (SupCon)~\cite{2020supcon} to learn representations that force samples from the same domain and the same live/spoof labels to form a compact cluster. 
To achieve the alignment, we devise a novel Projected Gradient optimization strategy based on Invariant Risk Minimization (PG-IRM) to regularize the 
live-to-spoof transition invariant to the domain variance. 
With normalization, the feature space is naturally divided into two symmetric half-spaces: one for live and one for spoof (see Fig.~\ref{fig:umap}).
Domain variations will manifest inside the half-spaces but 
have minimal impact to the live/spoof classifier. 

We summarize our contributions as three-fold:
\begin{compactitem}
    \item 
    We offer a new perspective for cross-domain FAS.
    Instead of removing the domain signal, we propose to maintain it and design the feature space based on separability and alignment;
    
    \item We first systematically exploit the domain-variant representation learning by combining contrastive learning and effectively optimizing invariant risk minimization (IRM) through the projected gradient algorithm for cross-domain FAS;
    
    \item We achieve state-of-the-art performance on widely-used cross-domain FAS benchmark, and provide in-depth analysis and insights on how separability and alignment lead to the performance boost.
\end{compactitem}

%% file: subtex/figure_2.tex
\begin{figure}[t]
\small\centering
\includegraphics[width=0.45\textwidth]{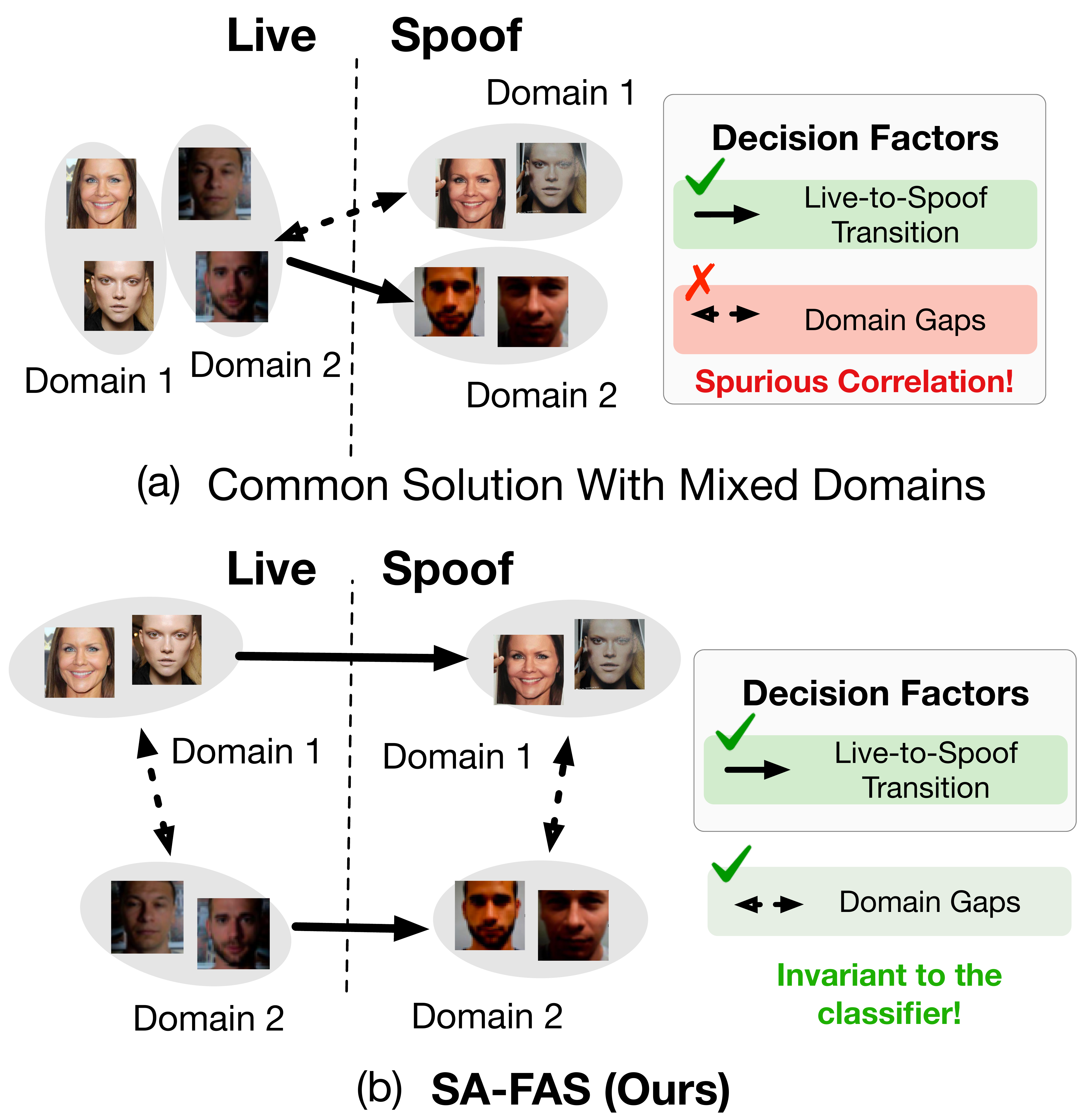}
\caption{\small
{\bf Cross-domain FAS:}
(a) Common FAS solutions aim to remove domain-specific signals and mix domains in one cluster. However, we empirically show domain-specific signals still exists in the feature space, and model might pick domain-specific signals as spurious correlation\protect\footnotemark~for classification.
(b) Our SA-FAS aims to retain domain signal. Specifically, we train a feature space with two critical properties: (1) \textbf{Separability}: Samples from different domains and
live/spoof classes are well-separated; (2) \textbf{Alignment}: Live-to-spoof transitions are aligned in the same direction for all domains. With these two properties, our method keeps the domain-specific signals invariant to the decision boundary.}
\label{fig:teaser}
\figvspace
\end{figure}

\footnotetext{In statistics, spurious correlation is a mathematical relationship in which multiple events or variables are associated but not causally related.}

%% file: sec2-prior.tex
\Section{Related Work}
\label{sec:related}
\vspace{\segsep}

\Paragraph{Face Anti-Spoofing} 
Face anti-spoofing attracts growing attention in several thriving directions.
Early works exploit spontaneous human behaviors (\eg, eye blinking, head motion)~\cite{kollreider2007real, pan2007eyeblink} or predefined movements (\eg, head-turning, expression changes)~\cite{chetty2010biometric}. 
Later, hand-crafted features are utilized to describe spoof patterns, \eg, LBP~\cite{boulkenafet2015face, freitas2012lbp}, HoG ~\cite{freitas2012lbp,yang2013face}
and SIFT~\cite{patel2016secure} features. 
Recently, deep neural networks have been applied to face anti-spoofing.
There are classification-based methods~\cite{yang2014learn,liu2019deep,wang2022patchnet}, regression-based methods~\cite{atoum2017face,liu2018learning,yu2020face, kim2019basn}, and generative models~\cite{jourabloo2018face,liu2020disentangling,wang2022ssan,liu2022spoof}. 
In addition, the vision transformer also shows promising performance in tackling FAS~\cite{george2021effectiveness,huang2022adaptive}.

\Paragraph{Cross-domain FAS} 
 Recently, several works explore learning FAS models from multiple domains that generalize to unseen ones. 
Some methods~\cite{zhou2022generative, li2018unsupervised, wang2020unsupervised, guo2022multi,unified-detection-of-digital-and-physical-face-attacks} require data from the target domain to adapt the model (\ie, domain adaptation), while others~\cite{shao2019multi,kim2021suppressing,saha2020domain,jia2020ssdg,wang2022ssan,noise-modeling-synthesis-and-classification-for-generic-object-anti-spoofing} learn shared features based on adversarial training and triplet loss (\ie, domain generalization).
A few methods~\cite{shao2020regularized,chen2021generalizable,wang2021self} explore meta-learning to simulate the domain shift at training time. 
Most previous works regard the domain-specific signals as a negative impact. Contrastively, our paper first systematically exploits the explicit usage of domain-specific signals by invariant risk minimization in cross-domain FAS.

\Paragraph{Domain-invariant Classifier}
Learning a domain-invariant classifier has always been the focus of machine learning for decades~\cite{van2018cpc,chen2020simclr,caron2020swav,he2019moco} and is also one of the keys to the success of domain generalization. Along this line, kernel-based methods~\cite{blanchard2021domain,muandet2013domain,grubinger2015domain,gan2016learning,li2018domain,ghifary2016scatter} propose to learn a domain-invariant kernel from the training data. Domain adversarial learning~\cite{li2018domain,ganin2015unsupervised,ganin2016domain,gong2019dlow,li2018deep,shao2019multi,mahfujur2019correlation,wang2022ssan,jia2020ssdg} adversarially trains the generator and discriminator while the generator is trained to fool the discriminator to learn domain invariant feature representations. Recently, Invariant Risk Minimization (IRM) and its variants~\cite{arjovsky2019irm,ahuja2021ibirm,krueger2021rex,mahfujur2019correlation,mitrovic2020representation,choe2020empirical,sonar2021invariant} seek to directly enforce the optimal classifier on top of the representation space to be the same across all domains. However, IRM is known to be hard to optimize and can fail in non-linear optimization~\cite{kamath2021doesirm,rosenfeld2020riskirm}. In this paper, we propose an equivalent objective (PG-IRM) which is easier to optimize and achieve strong performance. 


%% file: sec3-method.tex
\Section{Method}
\label{sec:method}
\vspace{\segsep}

We formally introduce the new learning framework, \textit{FAS with separability and alignment} (dubbed \textbf{SA-FAS}). The goal is to produce a feature space with two critical properties: 
(1) \textit{Separability}: We encourage samples from different domains and from different classes to be well-separated; 
(2) \textit{Alignment}: Live-to-spoof transition\footnote{The transition can be considered as a path in the high-D manifold. }  is aligned in the same direction for all domains.  
These two properties work jointly: 
\textit{separability} ensures the awareness of domain variance in the feature space; \textit{alignment} encourages the domain variance to be invariant to its live-vs-spoof hyperplane. 

This section is structured as follows:
Sec.~\ref{sec:setup} describes the problem setup, followed by the algorithm design of separability (Sec.~\ref{sec:sep}) and alignment (Sec.~\ref{sec:align}).
Finally, Sec.~\ref{sec:train_step} summarizes the training and inference processes.

\input{subtex/figure_3}

%

\SubSection{Problem Setup}
\label{sec:setup}
We start by defining the setting of the cross-domain FAS problem. We denote by $\mathcal{X}=\mathbb{R}^d$  the input space and 
$\mathcal{Y}=\{0 \text{ (live)}, 1 \text{ (spoof)}\}$ 
the output space. A learner is given access to a set of training data from $E$ domains $\mathcal{E} = \{e^{(1)}, e^{(2)}, .., e^{(E)}\}$ and is evaluated on test domain $e^*$. Let $e_i$ as the domain label for the $i$-th sample, we denote $\mathcal{D}=\{(\*x_i,y_i, e_i)\}_{i=1}^N$ drawn from an unknown joint data distribution $\mathcal{P}$ defined on $\mathcal{X} \times \mathcal{Y} \times \mathcal{E}$. 
Cross-domain FAS is a special binary classification problem to distinguish live and spoof faces from an unseen domain. The goal is to define a decision function:
\begin{align}
    f: \mathbf{x} \rightarrow \{0 \text{ (live)}, 1 \text{ (spoof)}\}, \nonumber
\end{align}
which classifies whether a sample $\*x$ from a new domain $e^*$ is live or spoof. 

In our network architecture, function $f$ consists of two components: (1) a deep neural network encoder $\phi:\mathcal{X} \rightarrow \mathbb{R}^m$ that maps the input $\bx$ to a $l_2$-normalized feature embedding $\*z = \phi(\bx)$; 
(2) a classifier (via a weight vector) $\beta:\mathbb{R}^m \rightarrow \mathbb{R}$ that maps the $m$-dimensional embedding $\*z$ to a scalar value, where a binary cross-entropy loss can be applied after using a sigmoid function. Because the true distribution of live/spoof data is unknown, the optimization commonly relies on an Empirical Risk Minimization (ERM).

\noindent \textbf{Remark on the terminology}: $\beta$ can be considered as a norm vector of the hyperplane separating live and spoof samples. In the remaining part of the paper, when we use ``\textbf{live-vs-spoof hyperplane}'' or ``\textbf{hyperplane}'', it has the same meaning as $\beta$. Note, ``live-to-spoof transition'' is an abstract procedure in the image space, while ``live-vs-spoof hyperplane'' refers to a concrete classifier in the feature space. 

\Paragraph{Preliminary on Empirical Risk Minimization (ERM):}
ERM principle~\cite{vapnik1991principles} is a ubiquitous strategy that merges data from all training domains and learns a predictor that minimizes an averaged training error.
Specifically,
\begin{align}
    & \mathcal{L}_{\textit{ERM}} = \min _{\phi, \beta} \frac{1}{|\mathcal{E}|} \sum_{e \in \mathcal{E}} \mathcal{R}^{e}(\phi, \beta), 
\label{eq:erm}
\end{align}
where the empirical risk function $\mathcal{R}^{e}(\phi, \beta)$ for a given environment $e$ is defined by:
$$
\mathcal{R}^{e}(\phi, \beta) \triangleq \mathbb{E}_{(\*x_i, y_i, e_i=e) \sim \mathcal{D}} \ell\left(f(\*x_i;\phi, \beta), y_i\right).
$$
Common choices of the loss function $\ell(\cdot,\cdot)$ include cross-entropy loss~\cite{jia2020ssdg} and $L_1$ regression loss~\cite{liu2018learning, george2019deep}.

However, if samples from different domains are mixed together, 
ERM can utilize the easiest difference (image resolution, blurriness, camera setting) to differentiate live \vs.~spoof. 
Such a classifier will undesirably leverage spurious correlations to make  live/spoof predictions~\cite{arjovsky2019irm}. Therefore, the naive strategy can hurt the generalization of the unseen domain. %
As shown in Fig.~\ref{fig:method}(a), ERM tends to fit all training data together and fails to learn a domain-invariant classifier with the mixed feature space. 
%

%
\SubSection{Separability}
\label{sec:sep}

We characterize the domain separability as supervised contrastive learning (dubbed SupCon)~\cite{2020supcon}, one of the latest developments for visual representation learning. Unlike other contrastive learning methods~\cite{chen2020simclr, chen2021exploring} that treat the augmented samples as a single class, SupCon aims to learn a representation space that gathers samples with the same labels while repelling samples from different ones. It naturally suits the need for the cross-domain FAS setting, since we treat samples with the same domain and with the same live/spoof label to form a cluster. 
%

Given a training mini-batch $\{\*x_i, y_i, e_i\}_{i=1}^b$, we augment~\cite{2020supcon} the mini-batch as $\{\tilde{\*x}_i, \tilde{y}_i, \tilde{e}_i\}_{i=1}^{2b}$, using two random augmentations $\tilde{\bx}_{2i}$ and $\tilde{\*x}_{2i-1}$ of inputs $\bx_i$, with $\tilde{y}_{2i-1} \!=\! \tilde{y}_{2i} \!=\! y_i$, $\tilde{e}_{2i-1} \!=\! \tilde{e}_{2i} \!=\! e_i$. These images are fed into the network, yielding $L_2$-normalized embeddings $\{\*z_i\}_{i=1}^{2b}$. The per-batch SupCon loss (separability loss) is defined as:
\begin{equation}
\mathcal{L}_\textit{sep}=\sum_{i=1}^{2b} \frac{-1}{|S(i)|} \sum_{j \in S(i)} \log \frac{\exp \left(\*z_{i} \cdot \*z_{j} / \tau\right)}{\sum_{t=1, t\neq i}^{2b}\exp \left(\*z_{i} \cdot \*z_{t} / \tau\right)},
\label{eq:supcon}
\end{equation}
where $\tau$ is a temperature parameter, $i$ is the index of a sample typically called the \textit{anchor}, $S(i) \!=\! \{j \!\in\!\{1,\ldots,2b\} \!:\! j\neq i, \tilde{y}_j = \tilde{y}_i, \tilde{e}_j = \tilde{e}_i\}$ is the index set of \emph{positive samples} 
that have the same live/spoof labels and belong to the same domain as the anchor $i$, and $|S(i)|$ is its cardinality. 
All the other samples in the mini-batch are referred to as \textit{negative samples}. 
Since positive samples are pulled together and negative samples are pushed apart, SupCon in Fig.~\ref{fig:method}(b) is capable of providing more distinguishable feature clusters for different domains and liveness classes, compared to a typical feature space learned by a vanilla ERM in Fig.~\ref{fig:method}(a). 
\SubSection{Alignment}
\label{sec:align}

Fig.~\ref{fig:method}(b) also shows that separability alone is not sufficient for improving domain generalization. 
The separated feature clusters can be located in any place in the feature space, and hence the domain-wise optimal hyperplane remains \textbf{variant}.
In this case, the global classifier can still undesirably incorporate the spurious correlation as the deciding factor as we show in Fig.~\ref{fig:teaser}.
To tackle this, we naturally investigate the following problem: 

\textit{How do we regularize a global live-vs-spoof hyperplane to align with domain-wise live-vs-spoof hyperplanes?}

We propose to formulate this problem as Invariant Risk Minimization (IRM)~\cite{arjovsky2019irm}, which aims to jointly optimize the feature space $\phi$ and the global live-vs-spoof hyperplane $\beta$, where $\beta$ is also optimal for each domain, shown in Fig.~\ref{fig:method}(c).

\input{subtex/figure_4}

\Paragraph{Preliminary on Invariant Risk Minimization (IRM):}
Specifically, the IRM objective can be formulated as the following constrained optimization problem:

\begin{align}
    & \min _{\phi, \beta^{*}} \frac{1}{|\mathcal{E}|} \sum_{e \in \mathcal{E}} \mathcal{R}^{e}(\phi, \beta^{*}) \rightarrow \mathcal{L}_{\textit{IRM}}  \label{eq:irm_target} \\ 
    & \textit{s.t.} \quad \beta^{*} \in \underset{\beta}{\arg\min} \mathcal{R}^{e}(\phi, \beta), \forall e \in \mathcal{E}.
\label{eq:irm_constrain}
\end{align}

Compared to the ERM~\eqref{eq:erm}, IRM enforces an additional constraint \eqref{eq:irm_constrain} to learn the domain-invariant hyperplanes. 
Specifically, if we define the domain-wise optimal hyperplane as 
$\beta_e \in {\arg\min  }_{\beta} \mathcal{R}^{e}(\phi, \beta)$.
A sufficient condition for constraint \eqref{eq:irm_constrain} to hold is $\beta_{e^{(1)}}=...=\beta_{e^{(E)}} = \beta^*$, which requires consistency between the globally optimal hyperplane and the domain-wise optimal hyperplanes.
However, IRM is known to be hard to solve~\cite{kamath2021doesirm,rosenfeld2020riskirm} due to the bi-level optimization nature of objective~(\ref{eq:irm_target}) and constraint~(\ref{eq:irm_constrain}). 

%

\Paragraph{Projected Gradient Optimization for IRM (PG-IRM):} 
We leverage Projected Gradient (PG) algorithm~\cite{nocedal1999numerical} to solve the non-trivial optimization objective~\eqref{eq:irm_target}, termed as PG-IRM. 
In PG-IRM, we propose to optimize multiple hyperplanes and converge them into a globally one via projected gradient. In Appendix~\ref{sec:proof_s1}, we provide detailed proof of PG-IRM objective being \textbf{equivalent} to IRM. Formally, the objective is rewritten as:
\begin{theorem}
\label{th:ourirm_objective}
\textbf{(PG-IRM objective)} For all $\alpha \!\in\! (0,1) $, the IRM objective is equivalent to the following objective: 

\begin{align}
    & \min _{\phi, \beta_{e^{(1)}}, ..., \beta_{e^{(E)}}} \frac{1}{|\mathcal{E}|} \sum_{e \in \mathcal{E}} \mathcal{R}^{e}(\phi, \beta_e) \rightarrow \mathcal{L}_{\textit{align}} \\ 
    &\text { s.t. }  \forall e \in \mathcal{E}, \exists \beta_e \in \Omega_e(\phi), \beta_e \in \Upsilon_{\alpha}(\beta_e), \nonumber 
\end{align}
where the parametric constrained set for each environment is simplified as 
$ \Omega_e(\phi) = \underset{\beta}{\arg \min } \mathcal{R}^{e}(\phi, \beta),$
and we define the \textbf{$\alpha$-adjacency set}: 
\begin{align}
    \Upsilon_{\alpha}(\beta_e) = \{\upsilon |& \underset{e' \in \mathcal{E} \backslash e}{\max}  \ \underset{\beta_{e'} \in \Omega_{e'}(\phi)}{\min}\|\upsilon - \beta_{e'}\|_2 \\
    & \le \alpha \underset{e' \in \mathcal{E} \backslash e}{\max}  \ \underset{\beta_{e'} \in \Omega_{e'}(\phi)}{\min}\|\beta_e - \beta_{e'}\|_2\}
    \label{eq:gamma_set}
\end{align}
\figvspace
\end{theorem}
Fig.~\ref{fig:illustration} shows the intuition of  the optimization process. 
For a $3$-domain case, PG-IRM starts with a shared feature space $\phi$ and $3$ separate hyperplanes $\beta_{e^{(1)}}$, $\beta_{e^{(2)}}$, $\beta_{e^{(3)}}$ for each domain (Fig.~\ref{fig:illustration}(b)). 
After each projected gradient descent, the hyperplanes move closer with feature space jointly updated (Fig.~\ref{fig:illustration}(c)). 
Upon convergence, $\beta_{e^{(1)}}, \beta_{e^{(2)}}, \beta_{e^{(3)}}$ become nearly identical (Fig.~\ref{fig:illustration}(d)), satisfying the IRM constraint $\beta^* \!=\! \beta_{e^{(1)}} \!=\! \beta_{e^{(2)}} \!=\! \beta_{e^{(3)}}$ for {\it all} domains.
We provide two main insights of our PG-IRM  algorithm (see more details in Appendix~\ref{sec:proof}):
\begin{compactenum}
    \item \textbf{Optimizing multiple  hyperplanes:} 
    Compared to the conventional IRM that optimizes a single hyperplane, it is easier to converge for PG-IRM that optimizes multiple hyperplanes (\ie, one for each domain).
    Shown in Fig.~\ref{fig:illustration}(a-b), for the same feature space from the intermediate optimization stage, the solution $\beta^*$ to conventional IRM may not exist and the optimization has to be terminated.
    In contrast, $\beta_{e^{(1)}}, ..., \beta_{e^{(E)}}$ \textbf{always} exists (Fig.~\ref{fig:illustration}(b)) which makes solving for multiple hyperplanes more viable.
    \item \textbf{Pushing hyperplanes to be closer:}
    To align $\beta_{e^{(1)}}$, $\beta_{e^{(2)}}$ and $\beta_{e^{(3)}}$,
    PG-IRM updates domain-wise hyperplanes by interpolating with other hyperplanes. It can be mathematically considered as projecting the parameters of a hyperplane into the $\alpha$-adjacency set $\Upsilon_{\alpha}(\beta_e)$ as we illustrated in Fig.~\ref{fig:proj}.

    \input{subtex/figure_projection.tex}

    \textbf{Remark (why PG is not applicable to IRM):}
    The PG algorithm can be infeasible for the conventional IRM, as the solution set for \eqref{eq:irm_constrain} can be \textbf{empty} and is thus non-projectable.
    Our PG-IRM objective in Eq.~\eqref{eq:gamma_set} contains a non-empty $\alpha$-adjacency set $\Upsilon_{\alpha}(\beta_e)$, and guarantees being projectable by simple linear interpolation. 

\end{compactenum}



%

\input{subtex/alg_main}

\input{subtex/tab_1}

\input{subtex/tab_2}

\SubSection{Training and inference}
\label{sec:train_step}

\Paragraph{Overall losses}
Considering the contrastive loss Eqn.~\eqref{eq:supcon}, the overall objective (dubbed as SA-FAS) can be written as: 

\begin{align}
    & \min _{\phi, \beta_{e^{(1)}}, ..., \beta_{e^{(E)}}} \mathcal{L}_{\textit{\textit{align}}} + \lambda \mathcal{L}_{\textit{sep}} \quad \rightarrow \mathcal{L}_{\textit{all}}  
    \label{eq:irm_constrain_overall} \\
    &\text { s.t. }  \forall e \in \mathcal{E}, \exists \beta_e \in \Omega_e(\phi), \beta_e \in \Upsilon_{\alpha}(\beta_e), \nonumber
\end{align}
where $\lambda$ is the coefficient for the loss term. The overall training pipeline is provided in Alg.~\ref{alg:main}.

\Paragraph{Inference}
At the inference stage, we use the mean hyperplane from $\beta_{e^{(1)}}, ..., \beta_{e^{(E)}}$ to get the final score. 
Specifically, the output is given by $$f(\*x) = \mathbb{E}_{e\in\mathcal{E}}[\beta_{e}^T\phi(\*x)].$$
Note that upon convergence, the cosine distance between any two of $\beta_{e^{(1)}}, ..., \beta_{e^{(E)}}$ is very close to 1, \ie, $\beta_{e^{(1)}} \!\approx\! ... \!\approx\! \beta_{e^{(E)}}$. 
This observation is verified in Appendix~\ref{sec:cosine_curve}, with an ablation (converged angles \vs different $\alpha$) in Appendix~\ref{sec:sensitivity}.

%% file: subtex/figure_3.tex
\begin{figure*}[t]
    \small\centering
    \includegraphics[width=0.95\textwidth]{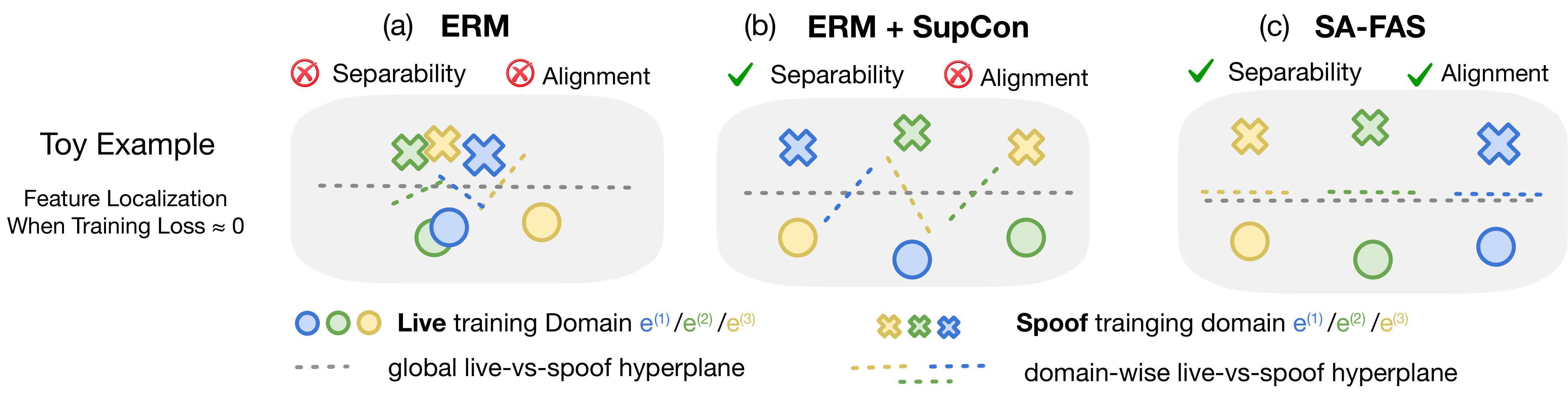}
    \caption{\small 
        {\bf Optimization objectives:}
        Illustration of feature space optimized by different objectives: (a) ERM, (b) ERM+SupCon, (c) SA-FAS (ours). 
        Circle/cross denotes live/spoof label; different colors indicate different domains. 
        A UMAP visualization for real data is provided in Appendix (Fig.~\ref{fig:method-sup}) to support the feature distribution shown in the toy example.
    }
    \label{fig:method}
\end{figure*}


%% file: subtex/figure_4.tex
\begin{figure*}[t]
    \small\centering
    \includegraphics[width=\textwidth]{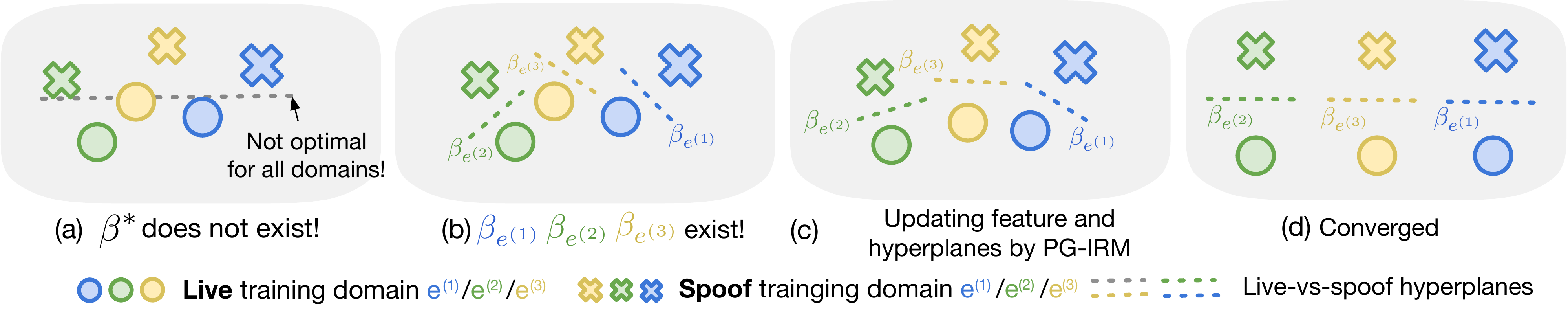}
    \vspace{-4ex}
    \caption{\small 
    {\bf PG-IRM optimization process:}
    An illustration of why a vanilla IRM can suffer from an infeasible solution (a), and how the proposed PG-IRM algorithm jointly updates the feature space and multiple hyperplanes towards convergence (b)-(d). }
    \label{fig:illustration}
\end{figure*}

%% file: subtex/figure_projection.tex
\begin{figure}[t!]
    \small\centering
    \includegraphics[width=.4\textwidth]{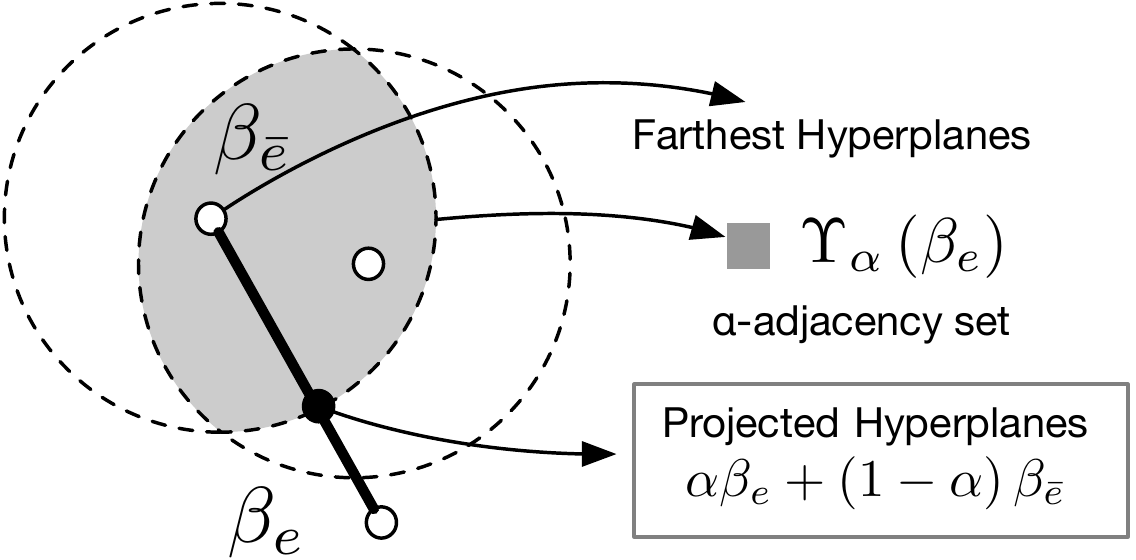}
    \caption{\small 
        {\bf Euclidean projection:}
        Illustration of Euclidean projection (solid black dot) to the $\alpha$-adjacency set $\Upsilon_\alpha\left(\beta_e\right)$. 
        Detailed proof and steps are provided in Alg.~\ref{alg:proj_grad_appendix} in Appendix~\ref{sec:proof_s2}. 
    }
    \figvspace
    \label{fig:proj}
\end{figure}

%% file: subtex/alg_main.tex
\begin{algorithm}[t]
\begin{algorithmic}[1]
  \State {\textbf{Input:}} Training data $\mathcal{D}=\{(\*x_i,y_i, e_i)\}_{i=1}^N$, network encoder $\phi$, classifiers $\beta_{e^{(1)}}, ... , \beta_{e^{(E)}}$, learning rate $\gamma$, alignment parameter $\alpha$, alignment starting epoch $T_a$.
        \For{$\text{t in 0, 1, ..., T}$}
        \State \textbf{Data Prep.}: Sample and augment a mini-batch.
        \State \textbf{Forward/Backward}:  Calculate gradient by $\mathcal{L}_{all}$.
        
        \For{$e \in \mathcal{E}$}
            \State $\tilde{\beta}^{t+1}_e = \beta^{t}_e - \gamma \nabla_{\beta^{t}_e} \mathcal{L}_{\textit{all}}$ \Comment{\textcolor{blue}{SGD}}
            \State select $\beta^{t}_{\bar{e}}$ with $\bar{e}  = \underset{e' \in \mathcal{E}  \backslash e }{\text{argmax}} \|\tilde{\beta}^{t+1}_e - \beta^{t}_{e'}\|_2$
            \State $\alpha' = 1 - \mathbf{1}_{t > T_a} (1 - \alpha)$ \Comment{\textcolor{blue}{$\alpha'$ is 1 when $t \le T_a$}}
            \State $\beta^{t+1}_e = \alpha'
            \tilde{\beta}^{t+1}_e + (1 - \alpha')  \beta^t_{\bar{e}}$ \Comment{\textcolor{blue}{Interpolation}}
        \EndFor
        \State Update $\phi^{t+1} = \phi^{t} - \gamma \nabla_{\phi^t}  \mathcal{L}_{\textit{all}}$. \Comment{\textcolor{blue}{Update encoder}} 
    \EndFor
\end{algorithmic}
\caption{Training pipeline for SA-FAS}
\label{alg:main}
\end{algorithm}

%% file: subtex/tab_1.tex
\begin{table*}[t]
    \begin{minipage}[c]{0.72\textwidth}
        
        \small\centering
        \scalebox{0.85}{
            \begin{tabular}{rrrrrrrrr} \toprule
            \multicolumn{1}{c}{\multirow{2}{*}{\textbf{Method ($\%$)}}} & \multicolumn{2}{c}{\textbf{OCI$\rightarrow$M}} & \multicolumn{2}{c}{\textbf{OMI$\rightarrow$C}} & \multicolumn{2}{c}{\textbf{OCM$\rightarrow$I}} & \multicolumn{2}{c}{\textbf{ICM$\rightarrow$O}} \\
            \multicolumn{1}{c}{} & \multicolumn{1}{c}{\textbf{HTER }$\downarrow$} & \multicolumn{1}{c}{\textbf{AUC} $\uparrow$} & \multicolumn{1}{c}{\textbf{HTER  $\downarrow$}} & \multicolumn{1}{c}{\textbf{AUC $\uparrow$}} & \multicolumn{1}{c}{\textbf{HTER}  $\downarrow$} & \multicolumn{1}{c}{\textbf{AUC} $\uparrow$} & \multicolumn{1}{c}{\textbf{HTER}  $\downarrow$} & \multicolumn{1}{c}{\textbf{AUC}  $\uparrow$}
             \\ \midrule
            MMD-AAE~\cite{li2018domain} & 27.08 & 83.19 & 44.59 & 58.29 & 31.58 & 75.18 & 40.98 & 63.08 \\
            MADDG~\cite{shao2019multi} & 17.69 & 88.06 & 24.50 & 84.51 & 22.19 & 84.99 & 27.98 & 80.02 \\
            SSDG-M~\cite{jia2020ssdg} & 16.67 & 90.47 & 23.11 & 85.45 & 18.21 & 94.61 & 25.17 & 81.83 \\
            DR-MD-Net~\cite{wang2020cross} & 17.02 & 90.10 & 19.68 & 87.43 & 20.87 & 86.72 & 25.02 & 81.47 \\
            RFMeta~\cite{shao2020regularized} & 13.89 & 93.98 & 20.27 & 88.16 & 17.30 & 90.48 & 16.45 & 91.16 \\
            NAS-FAS~\cite{yu2020fas} & 19.53 & 88.63 & 16.54 & 90.18 & 14.51 & 93.84 & 13.80 & 93.43 \\
            D2AM~\cite{chen2021generalizable} & 12.70 & 95.66 & 20.98 & 85.58 & 15.43 & 91.22 & 15.27 & 90.87 \\
            SDA~\cite{wang2021self} & 15.40 & 91.80 & 24.50 & 84.40 & 15.60 & 90.10 & 23.10 & 84.30 \\
            DRDG~\cite{liu2021dual} & 12.43 & 95.81 & 19.05 & 88.79 & 15.56 & 91.79 & 15.63 & 91.75 \\
            ANRL~\cite{liu2021adaptive} & 10.83 & 96.75 & 17.83 & 89.26 & 16.03 & 91.04 & 15.67 & 91.90 \\
            SSAN-M~\cite{wang2022ssan} & 10.42 & 94.76 & 16.47 & 90.81 & 14.00 & 94.58 & 19.51 & 88.17 \\
            SSDG-R~\cite{jia2020ssdg} & 7.38 & 97.17 & 10.44 & 95.94 & 11.71 & 96.59 & 15.61 & 91.54 \\
            SSAN-R~\cite{wang2022ssan} & 6.67 & \textbf{98.75} & 10.00 & \textbf{96.67} & 8.88 & 96.79 & 13.72 & 93.63 \\
            PatchNet~\cite{wang2022patchnet} & 7.10 & 98.46 & 11.33 & 94.58 & 13.40 & 95.67 & 11.82 & 95.07 \\
            \methodname (Ours) & \textbf{5.95} & 96.55 & \textbf{8.78} & 95.37 & \textbf{6.58}  & \textbf{97.54} & \textbf{10.00} & \textbf{96.23}  \\ \bottomrule
            \end{tabular}
        }
    \end{minipage}
    \hfill
    \begin{minipage}[t]{0.28\textwidth}
        \vspace*{-40pt}
        \caption{
            \small 
            {\bf Comparisons with SoTA methods:} 
            Cross-domain face anti-spoofing is evaluated among four popular benchmark datasets: CASIA (\textbf{C}), Idiap Replay (\textbf{I}), MSU-MFSD (\textbf{M}), and Oulu-NPU (\textbf{O}). 
            Methods are compared at their best performance following the commonly used evaluation process \cite{jia2020ssdg}. 
            $\uparrow$ indicates larger values are better, and $\downarrow$ indicates smaller values are better.
        }
        \label{tab:best}
    \end{minipage}
    \vspace{-1ex}
\end{table*}

%% file: subtex/tab_2.tex
\begin{table*}[t]
\small \centering
\scalebox{0.81}{
\begin{tabular}{lllll} \toprule
\multicolumn{1}{c}{\multirow{2}{*}{\textbf{Method ($\%$)}}} & \multicolumn{1}{c}{\textbf{OCI$\rightarrow$M}} & \multicolumn{1}{c}{\textbf{OMI$\rightarrow$C}} & \multicolumn{1}{c}{\textbf{OCM$\rightarrow$I}} & \multicolumn{1}{c}{\textbf{ICM$\rightarrow$O}} \\
& \multicolumn{1}{c}{\textbf{HTER}$\downarrow$ /\textbf{AUC}$\uparrow$/\textbf{TPR95}$\uparrow$}  & \multicolumn{1}{c}{\textbf{HTER}$\downarrow$ /\textbf{AUC}$\uparrow$/\textbf{TPR95}$\uparrow$}  & \multicolumn{1}{c}{\textbf{HTER}$\downarrow$ /\textbf{AUC}$\uparrow$/\textbf{TPR95}$\uparrow$}  &
\multicolumn{1}{c}{\textbf{HTER}$\downarrow$ /\textbf{AUC}$\uparrow$/\textbf{TPR95}$\uparrow$}
 \\ \midrule
SSDG-R~\cite{jia2020ssdg} & 14.65 $ ^{{1.21}} $ / 91.93	$ ^{{1.35}} $  / 53.68  $ ^{{2.56}} $
& 28.76	$ ^{{0.89}} $ / 80.91	$ ^{{1.10}} $ / 41.47 $ ^{{2.68}} $ 
& 22.84	$ ^{{1.14}} $ / 78.67	$ ^{{1.31}} $  / 50.80  $ ^{{5.95}} $
& 15.83	$ ^{{1.29}} $ / 92.13	$ ^{{0.96}} $ / 66.54 $ ^{{4.00}} $ \\
SSAN-R~\cite{wang2022ssan} & 21.79 $ ^{{3.68}} $ /  84.06   $ ^{{3.78}} $  / 51.91  $ ^{{4.28}} $
&  26.44   $ ^{{2.91}} $ /  78.84   $ ^{{2.83}} $ / 45.36 $ ^{{4.29}} $ 
&  35.39   $ ^{{8.04}} $ /  70.13   $ ^{{9.03}} $  / 64.00  $ ^{{2.70}} $
&  25.72   $ ^{{3.74}} $ /  79.37   $ ^{{4.69}} $ / 36.75  $ ^{{5.19}} $ \\
PatchNet~\cite{wang2022patchnet} & 25.92   $ ^{{1.13}} $ /  83.43   $ ^{{0.87}} $ / 38.75 $ ^{{8.31}} $
&  36.26   $ ^{{1.98}} $ /  71.38   $ ^{{1.89}} $ / 19.22 $ ^{{3.85}} $ 
&  29.75   $ ^{{2.76}} $ /  80.53   $ ^{{1.35}} $  / 54.25  $ ^{{2.18}} $
&  23.49   $ ^{{1.80}} $ / 84.62   $ ^{{1.92}} $ / 39.39  $ ^{{6.83}} $ \\ \midrule 
SA-FAS (Ours) 
& \textbf{14.36}	$ ^{{1.10}} $ / \textbf{92.06}	$ ^{{0.53}} $ / \textbf{55.71}	$ ^{{4.82}} $ 
& \textbf{19.40}	$ ^{{0.66}} $ / \textbf{
88.69}	$ ^{{0.67}} $ / \textbf{50.53}	$ ^{{3.60}} $ 
& \textbf{11.48}	$ ^{{1.10}} $ / \textbf{
95.74}	$ ^{{0.55}} $ / \textbf{77.05}	$ ^{{3.26}} $ 
& \textbf{11.29}	$ ^{{0.32}} $ / \textbf{95.23}	$ ^{{0.24}} $ / \textbf{73.38} $ ^{{1.64}} $ \\
\bottomrule
\end{tabular}}
\figvspace
\caption{\small 
{\bf Evaluation upon convergence:}
Evaluation of cross-domain face anti-spoofing among CASIA (\textbf{C}), Idiap Replay (\textbf{I}), MSU-MFSD
(\textbf{M}), and Oulu-NPU (\textbf{O}) databases. Methods are compared at their mean/std performance based on the last 10 epochs. 
}
\label{tab:mean}
\vspace{-1.5ex}
\end{table*}

%% file: sec4-exp.tex
\Section{Experiments}
\label{sec:exp}

\SubSection{Experimental setups}

\Paragraph{Datasets and protocols} 
We  evaluate on four widely used  datasets: \texttt{Oulu-NPU} (\textbf{O})~\cite{boulkenafet2017oulu}, \texttt{CASIA} (\textbf{C})~\cite{zhang2012casia}, \texttt{Idiap Replay attack} (\textbf{I})~\cite{chingovska2012replay}, and \texttt{MSU-MFSD} (\textbf{M})~\cite{wen2015msu}. 
Following  prior works, we treat each dataset as one domain and apply the leave-one-out test protocol to evaluate their cross-domain generalization. 
Specifically, we refer \textbf{OCI$\rightarrow$M} to be the protocol that trains on \texttt{Oulu-NPU}, \texttt{CASIA}, \texttt{Idiap Replay attack} and tests on  \texttt{MSU-MFSD}. \textbf{OMI$\rightarrow$C}, \textbf{OCM$\rightarrow$I} and \textbf{ICM$\rightarrow$O} are defined in a similar fashion. 

\Paragraph{Implementation details} 
The input images are cropped using MTCNN~\cite{zhang2016mtcnn} and resized to 256$\times$256. 
For fair comparisons with SoTA methods \cite{jia2020ssdg,wang2022ssan,wang2022patchnet}, we use the same ResNet-18 backbone.
We train the network with SGD optimizer and an initial learning rate of 5e\text{-}3, which is decayed by 2 at epoch 40 and 80 and the total training epoch is 100 in most set-ups\footnote{Due to the smaller training data size of \textbf{ICM}, we let the  \textbf{ICM$\rightarrow$O} to train for 300 epochs and decay at epoch 120 and 240. }.
We set the weight decay as 5e\text{-}4 and the batch size as 96 for each training domain. 
For \methodname hyperparameters, we set $\alpha\!=\!0.995$, $\lambda\!=\! 0.1$ and $T_a \!=\! 20$. 

\Paragraph{Evaluation metrics} We evaluate the model performance using three standard metrics: Half Total Error Rate
(HTER), Area Under Curve (AUC), and True Positive Rate (TPR95) at a False Positive Rate
(FPR) 5\%. 
While HTER and AUC assess the theoretical performance, TPR at a certain FPR is adept at
reflecting how well the model performs in practice. 

\SubSection{Cross-domain performance}

Tab.~\ref{tab:best} summarizes our comparison with an extensive collection of recent studies, including SoTA methods: \texttt{PatchNet}~\cite{wang2022patchnet}, \texttt{SSAN}~\cite{wang2022ssan} and \texttt{SSDG}~\cite{jia2020ssdg}.
\methodname outperforms the rivals by a significant margin on cross-domain FAS benchmarks.
In particular, we improve upon the best baseline \cite{wang2022ssan} by 2.30\% in HTER in the setting \textbf{OCM$\rightarrow$I}, which is more than \textbf{25\%} improvement.

\Paragraph{Comparison upon convergence} 
Note that the performance in Tab.~\ref{tab:best} follows the convention in \cite{jia2020ssdg}, which is reported on the training snapshot (\eg, epoch 16) with the lowest test error. 
While this setting may manifest the best performance from the model, the results can significantly fluctuate on the test set and hard to reflect the generalization performance when a test set is unavailable (shown in Appendix~\ref{sec:whyfairsetting}).
To provide a more fair setting, we propose to report the average performance from the \textbf{last 10 epochs} upon convergence.
In our case, the stopping criterion is either (1) the binary classification loss for live/spoof is smaller than 1e\text{-}3 for consecutive 10 epochs, or (2) the epoch number reaches max limit, whichever comes first.

In Tab.~\ref{tab:mean}, we compare with SoTA methods in this setting, and provide three key observations: 
(1) The numbers are way worse than the ones in Tab.~\ref{tab:best} across all methods, indicating the best model selected by conventional lowest test errors \cite{jia2020ssdg} has large randomness.
This also shows that cross-domain FAS is far less-solved than expected.
(2) The standard deviation in Tab.~\ref{tab:mean} denotes how stable each method performs.
Most methods can converge to a relatively stable status, while methods with an adversarial loss (\eg, \cite{wang2022ssan}) have a relatively larger standard deviation, indicating adversarial loss might trigger more unstable training.
(3) In our setting, \methodname still largely outperforms SoTA \cite{wang2022patchnet,wang2022ssan,jia2020ssdg}, which further validates the superiority of our method.
Our method is also the most stable compared to SoTA, with the smallest standard deviation.
We proceed by analyzing why traditional methods are less favorable in cross-domain FAS and why our methods perform better.

\Section{Ablation and Discussion}
\input{subtex/tab_3}

\SubSection{Effectiveness of loss components}  
Our overall objective function Eq.~\eqref{eq:irm_constrain_overall} consists of two parts: (a) Separability loss ($\mathcal{L}_{sep}$) for feature space; and (b) Alignment loss ($\mathcal{L}_{align}$) for regularizing the classifier. We ablate the contribution of each component in Tab.~\ref{tab:abl_main}. 

\Paragraph{Separability loss}
We consider the two most common strategies used in the contrastive learning community (\ie, \texttt{SimCLR}~\cite{chen2020simclr}, \texttt{SimSiam}~\cite{chen2020simclr}) and one in face recognition (\ie, \texttt{Triplet loss}~\cite{schroff2015facenet}). We also provide the comparison of SupCon with the clustering policy used in SSDG~\cite{jia2020ssdg}\footnote{SSDG assumes live samples in all domains form one cluster, and spoof samples in each domain respectively form the other three clusters.}. 
All losses are directly applied on the penultimate layer's feature $\phi(\*x)$ with the same hyper-parameters, and all final classifications are supervised by ERM.
We observe that the SupCon loss used in our framework outperforms other rivals.
This validates the effectiveness of the domain-wise separable feature space for cross-domain FAS.

\Paragraph{Alignment loss}
IRM objective is known to be hard to optimize. 
Other than the proposed PG-IRM, existing works  \texttt{IRM-v1}~\cite{arjovsky2019irm}, \texttt{IB-IRM}~\cite{ahuja2021ibirm}, and \texttt{VRex}~\cite{krueger2021rex} alternatively consider a Lagrangian form:
\begin{equation}
\min _{\phi, \beta^{*}} \frac{1}{|\mathcal{E}|} \sum_{e \in \mathcal{E}}\left[\mathcal{R}^{e}(\phi, \beta^{*})+\lambda\left\|\nabla_{\beta^{*}} \mathcal{R}^{e}(\phi, \beta^{*})\right\|_{2}^{2}\right].
\label{eq:irm_langran}
\end{equation}
In Tab.~\ref{tab:abl_main}, we compare PG-IRM with the baseline ERM as well as other IRM alternatives, and our method shows a better overall performance. This shows further evidence that the Lagrangian penalty term can be ineffective, especially in the non-linear case \cite{rosenfeld2020riskirm, kamath2021doesirm}. In comparison, PG-IRM optimizes the IRM objective directly with Projected Gradient, which clearly distinguishes it from existing methods. 

Overall, the ablation studies suggest all components in our framework are indispensable to enhancing the generalization ability of cross-domain spoof detection.
\input{subtex/figure_corr}
\input{subtex/figure_umap}
\input{subtex/figure_umap_cmp_da}

\SubSection{Separability and alignment analysis}  

SA-FAS aims to produce a feature space with two critical properties: Separability and Alignment. 
In this section, we empirically investigate if these two properties can lead to a better generalization performance. Specifically, we provide two corresponding measures based on the learned classifiers and the extracted feature vector $\*z$ of samples from the \textbf{test} domain. 
We define the separability score as:
$$S_{sep} =  1 - \cos(\mathbb{E}_{spoof}[\*z], \mathbb{E}_{live}[\*z]), $$
where we measure the cosine angle between the center of live/spoof features. 
A separated feature space naturally leads to a small cosine value and thus a larger $S_{sep}$ score. 
For the alignment score, we define:
\vspace{-2ex}
\vspace{\baselineskip}
\begin{equation}
    S_{align} = \mathbb{E}_{e \in \mathcal{E}}[ \cos(\beta_e, \tikzmarknode{oracle}{\highlight{blue}{$\mathbb{E}_{spoof}[\*z] - \mathbb{E}_{live}[\*z])$}})],
\end{equation}
where the trajectory from the center of spoof to live is treated as an \textit{oracle} vector, which we measure how close it is with the norm vector of $\beta$ of the learned hyperplane.
\begin{tikzpicture}[overlay,remember picture,>=stealth,nodes={align=left,inner ysep=1pt},<-]
    \path (oracle.north) ++ (0,0.5em) node[anchor=south west,color=blue!67] (scalep){\textit{oracle vector}};
    \draw [color=blue!87](oracle.north) |- ([xshift=-0.3ex,color=blue]scalep.south east);
\end{tikzpicture}

With the measure of two properties, we show their correlation to the generalization performance (\ie, AUC) in Fig.~\ref{fig:corr_all}.
We see that $S_{sep}$ and $S_{align}$ are positively related to their test performance. 
It validates that these two properties are beneficial for a domain-invariant classifier. 
Specifically, Fig.~\ref{fig:corr_all}(a) compares the setting with and without PG-IRM. Using PG-IRM leads to a higher alignment score and AUC, which verifies that PG-IRM can better align the live-vs-spoof hyperplanes for the unseen domain and improve the generalization ability. Similarly, Fig.~\ref{fig:corr_all}(b) compares the setting with and without SupCon. The results validate that SupCon can lead to better separability in the feature space which benefits the classification. 

\SubSection{UMAP visualization}  
Fig.~\ref{fig:umap} first provides UMAP~\cite{mcinnes2018umap} visualization of \methodname feature space from the penultimate layer. We see that the hyperplane between live samples and spoof samples is consistent across different training domains and also transferable to unseen test domains. 
For instance, in the setting of \texttt{OMI$\rightarrow$C}, 
the test live samples in blue circles can be separated from the test spoof samples in blue crosses by the hyperplane.  Another interesting finding is that some \texttt{CASIA} samples in blue are closer to \texttt{OULU} with \textbf{high resolution} and some are closer to \texttt{MSU} or \texttt{REPLAY} with \textbf{low resolution}, which reflects the fact that \texttt{CASIA} is a mixed dataset with both low and high resolution images. 
These findings validate that the domain gap (resolution) manifests in a way that is invariant to the live-vs-spoof hyperplane. 

Beyond numerical and visual results, the superiority of domain-variant feature space can also be validated by theoretical support. Specifically, the estimated error bound for binary
classification in domain generalization~\cite{blanchard2021domain} becomes larger if $(M, n)$ is replaced with $(1, Mn)$, where $M$ is the domain number and $n$ is the training set size per domain. 
It indicates that separately training datasets from different domains is better than pooling them into one mixed dataset.

\Paragraph{DANN~\cite{ganin2016dann} and SSDG~\cite{jia2020ssdg} visualization} 
We also compare the feature space of methods that aim to remove domain-specific signals from its feature representation.
DANN~\cite{ganin2016dann} leverages the adversarial loss to encourage the backbone to provide a domain-invariant feature. 
Fig.~\ref{fig:cmp_da_umap}(a) shows that the domain gap yet still broadly exists, especially for the test data from an unseen domain, which backfires on the generalizability of the classifier.
Similarly, SSDG~\cite{jia2020ssdg} learns a partial domain-invariant feature space where all live samples are clustered in one group while spoof samples are kept to be domain-dispersed. 
Although the degradation direction aligns better between train and test, compared to DANN, the domain gap still exists for live training samples as shown in Fig.~\ref{fig:cmp_da_umap}(b). 
These findings further validate the necessity of regularizing the live-vs-spoof hyperplanes to be consistent across different domains.

%% file: subtex/tab_3.tex
\begin{table}[t]
\small \centering
\scalebox{0.85}{
\begin{tabular}{llll} \toprule
\multicolumn{1}{c}{\multirow{2}{*}{\textbf{Method}}} & \multicolumn{3}{c}{\textbf{Average}} \\
& \multicolumn{1}{c}{\textbf{HTER}$\downarrow$} & \multicolumn{1}{c}{\textbf{AUC}$\uparrow$} & \multicolumn{1}{c}{\textbf{TPR95}$\uparrow$}
 \\   \midrule
 SimCLR~\cite{chen2020simclr}  & 22.53 $ ^{1.31} $ & 84.42 $ ^{1.04} $ & 51.14 $ ^{3.44} $  \\
SimSiam~\cite{chen2021exploring}  & 18.89 $ ^{0.97} $ & 89.93 $ ^{0.80} $ & 56.62 $ ^{2.88} $  \\
  Triplet~\cite{schroff2015facenet}  & 18.75 $ ^{2.31} $ & 88.11 $ ^{2.30} $ & 50.53 $ ^{8.76} $  \\
SupCon (SSDG)~\cite{jia2020ssdg}  & 17.91 $ ^{1.05} $ & 90.10 $ ^{0.68} $ & 61.98 $ ^{2.87} $  \\
SupCon~\cite{2020supcon}  & 17.03 $ ^{1.73} $ & 90.68 $ ^{1.29} $ & 56.72 $ ^{5.06} $  \\
 \midrule 
ERM   & 17.22 $ ^{1.26} $ & 90.21 $ ^{1.38} $ & 58.62 $ ^{3.77} $  \\
DANN~\cite{ganin2016dann}  & 17.93 $ ^{1.02} $ & 90.66 $ ^{0.56} $ & 58.66 $ ^{3.14} $  \\
  IRM-v1~\cite{arjovsky2019irm}   & 17.41 $ ^{0.77} $ & 91.16 $ ^{0.52} $ & 60.98 $ ^{2.10} $  \\
 VREx~\cite{krueger2021out}   & 25.02 $ ^{1.92} $ & 80.65 $ ^{2.20} $ & 45.12 $ ^{3.78} $  \\
IB-IRM~\cite{ahuja2021ibirm}   & 17.57 $ ^{0.74} $ & 91.71 $ ^{0.51} $ & 62.16 $ ^{2.35} $  \\ 
PG-IRM (Ours)   & 15.58 $ ^{0.96} $ & 92.03 $ ^{0.62} $ & 63.31 $ ^{2.59} $  \\  \midrule
SA-FAS (Ours)  & \textbf{14.25} $ ^{0.79} $ & \textbf{92.93} $ ^{0.49} $ & \textbf{64.16} $ ^{3.33} $  \\
\bottomrule
\end{tabular}}
\figvspace
\caption{\small 
{\bf Ablation study:}
The averaged performance is computed over all four cross-domain settings.}
\label{tab:abl_main}
\vspace{-2ex}
\end{table}

%% file: subtex/figure_corr.tex
\begin{figure}[t]
    \small\centering
    \includegraphics[width=0.5\textwidth]{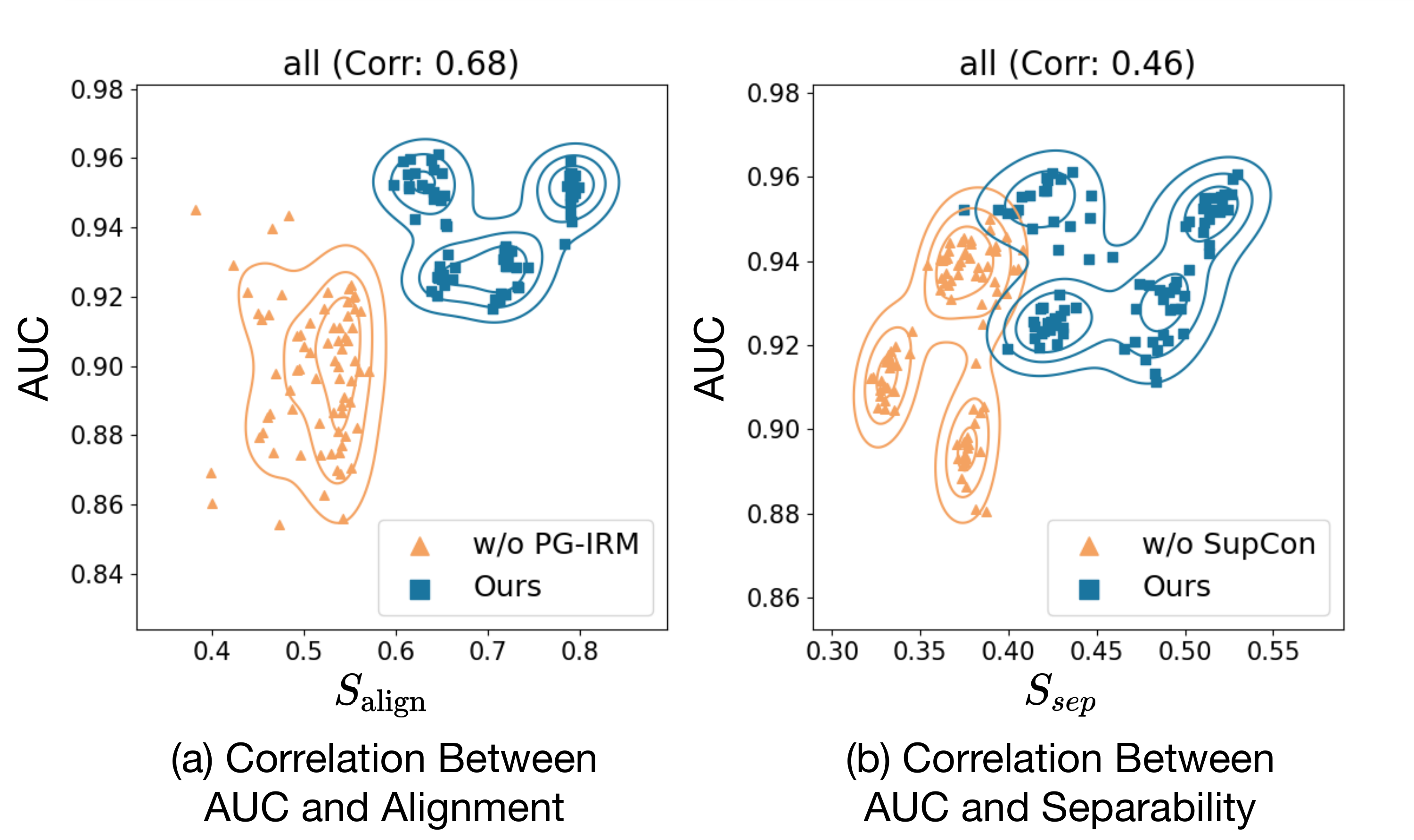}
    \vspace{-4ex}
    \caption{\small 
    {\bf Correlation of performance and SA-FAS:}:
    Correlation between the test performance AUC and two properties measure ($S_{align}/S_{sep}$). Each dot represents one snapshot during the training stage in all four cross-domain settings. We provide separate figures for each setting in Appendix (Fig.~\ref{fig:more_corr}).}
    \label{fig:corr_all}
    \vspace{-0.5ex}
\end{figure}

%% file: subtex/figure_umap.tex
\begin{figure*}[t]
    \small\centering
    \includegraphics[width=1.0\textwidth]{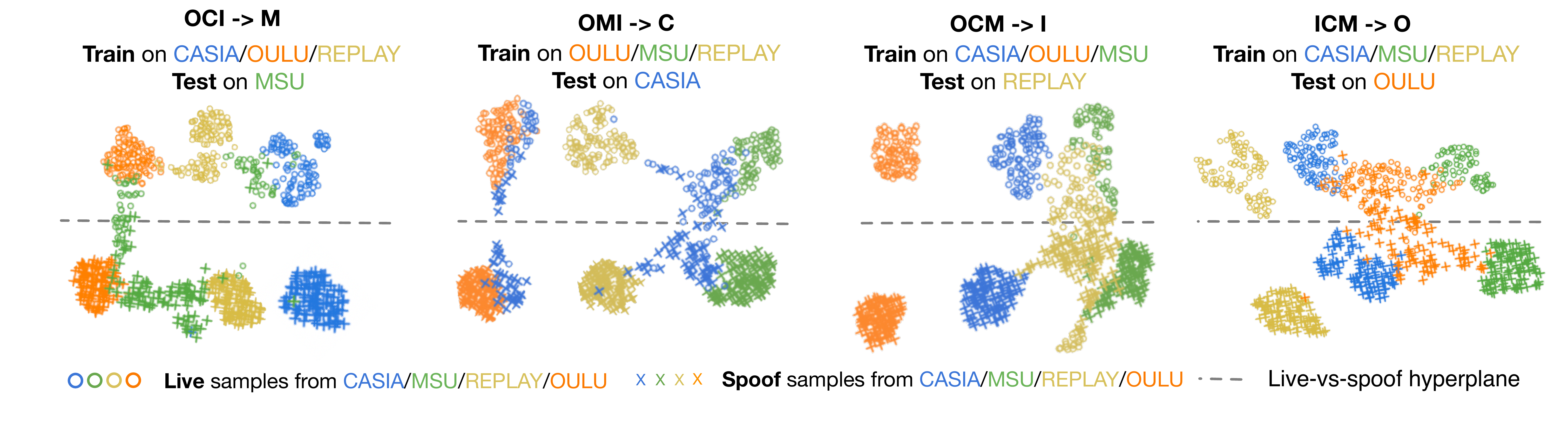}
    \vspace{-5ex}
    \caption{\small 
    {\bf Feature learned in different domains}:
    UMAP~\cite{mcinnes2018umap} projection of the penultimate layer of ResNet-18 trained with \methodname in the cross-test setting of face anti-spoofing datasets. 
    The dotted line shows the decision boundary derived from training samples in 2D space.}
    \figvspace
    \label{fig:umap}
\end{figure*}

%% file: subtex/figure_umap_cmp_da.tex
\begin{figure}[t]
    \small\centering
    \includegraphics[width=.46\textwidth]{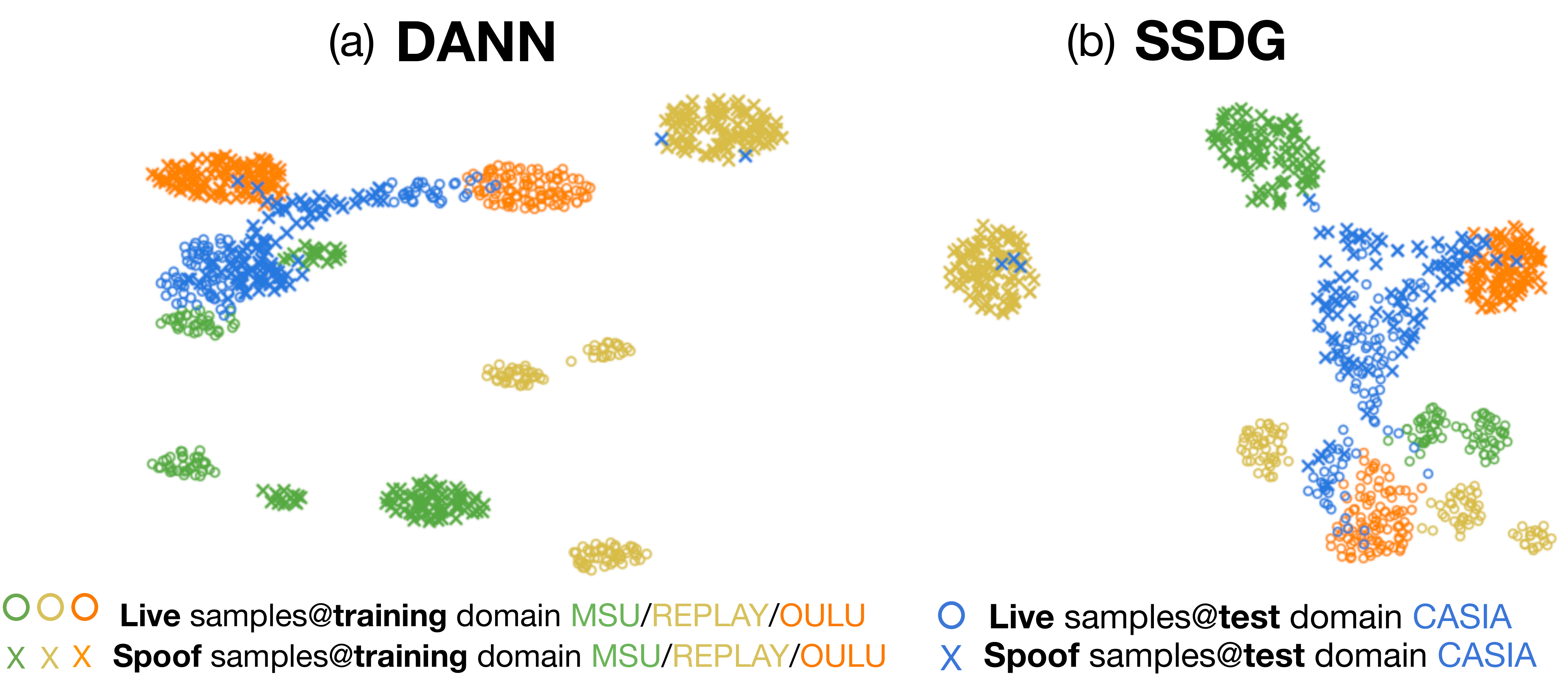}
    \figvspace
    \caption{\small
    {\bf Features of DANN \vs~SSDG:}
    UMAP \cite{mcinnes2018umap} visualization of the penultimate layer of ResNet-18 trained with DANN \cite{ganin2016dann} and SSDG \cite{jia2020ssdg} in the cross-test setting of OMI$\rightarrow$C.}
    \vspace{-0.3cm}
    \label{fig:cmp_da_umap}
\end{figure}

%% file: sec5-conclusion.tex
\Section{Conclusion}

This paper provides a new learning framework \methodname that learns domain-variant features but domain-invariant decision boundaries for cross-domain FAS. Our framework is naturally motivated, which facilitates invariant decision boundaries and learning distinguishable representations. We provide important theoretical insights that IRM objectives can be equivalently optimized by the PG with an alternative objective. Experiments show that \methodname can notably improve performance compared to the current best methods, establishing state-of-the-art. We also discuss the limitation of our work in Appendix~\ref{sec:limit}. 
We hope this paper will inspire more future works in incorporating domain-specific signals in FAS feature representation, and also extending this idea to broader domain generalization tasks.

\newpage

%% file: appendix.tex
\clearpage
\newpage
\appendix

\input{subtex/figure_method_sup}


\section{Detailed Proof}
\label{sec:proof}
In the main paper, we propose to use a project gradient algorithm to efficiently optimize the hard IRM objective. In this section, we provide a formal proof composed of two main steps: 
(1) We show in Section~\ref{sec:proof_s1} that the original IRM objective is equivalent to the PG-IRM objective shown in Theorem~\ref{th:ourirm_objective}. (2) In Section~\ref{sec:proof_s2}, we show that the PG-IRM objective can be efficiently optimized by the project gradient descent algorithm illustrated in Alg.~\ref{alg:proj_grad_appendix}.

\subsection{PG-IRM objective is equivalent to IRM}
\label{sec:proof_s1}
As a recap of our learning setting, a learner is given access to a set of training data from $E$ environments $\mathcal{E} = \{e^{(1)}, e^{(2)}, .., e^{(E)}\}$ and the IRM objective is the following constrained optimization problem:

\begin{align}
    & \min _{\phi, \beta^{*}} \frac{1}{|\mathcal{E}|} \sum_{e \in \mathcal{E}} \mathcal{R}^{e}(\phi, \beta^{*}) \quad  \label{eq:irm_target_appendix} \\ 
    & \text { s.t. } \quad \beta^{*} \in \underset{\beta}{\arg\min  } \mathcal{R}^{e}(\phi, \beta) \quad \forall e \in \mathcal{E},
\label{eq:irm_constrain_appendix}
\end{align}
where the risk function for a given domain/distribution e is:

$$
\mathcal{R}^{e}(\phi, \beta) \doteq \mathbb{E}_{(\*x_i, y_i, e_i=e) \sim \mathcal{D}} \ell\left(f(\*x_i;\phi,\beta), y_i\right).
$$

\begin{theorem*}  (Recap of Theorem~\ref{th:ourirm_objective})
For all $\alpha \in (0,1) $, the IRM objective is equivalent to the following objective: 

\begin{align}
    & \min _{\phi, \beta_{e^{(1)}}, ..., \beta_{e^{(E)}}} \frac{1}{|\mathcal{E}|} \sum_{e \in \mathcal{E}} \mathcal{R}^{e}(\phi, \beta_e) \quad  \label{eq:irm_target_sepirm} \\ 
    &\text { s.t. }  \forall e \in \mathcal{E}, \exists \beta_e \in \Omega_e(\phi), \beta_e \in \Upsilon_{\alpha}(\beta_e),
\label{eq:irm_constrain_sepirm}
\end{align}
where the parametric constrained set for each environment is simplified as 
$$ \Omega_e(\phi) = \underset{\beta}{\arg \min } \mathcal{R}^{e}(\phi, \beta),$$
and we define 
\begin{align}
\begin{split}
    & \Upsilon_{\alpha}(\beta_e) = \{\upsilon | \underset{\forall e' \in \mathcal{E}  \backslash e, \beta_{e'} \in \Omega_{e'}(\phi)}{\min}\|\upsilon - \beta_{e'}\|_2 \\
    & \le \alpha \underset{\forall e' \in \mathcal{E}  \backslash e, \beta_{e'} \in \Omega_{e'}(\phi)}{\min}\|\beta_e - \beta_{e'}\|_2\}
\end{split}
\end{align}
\label{th:pgirm_sup}
\end{theorem*}

\begin{proof}

The constraint (\ref{eq:irm_constrain_appendix}) means that the $\beta^*$ is the optimal linear classifier at all environments, which is equivalent to saying that $\beta^*$ lies in the joint of the optimal solution set in each environment.
Equivalently, we can formularize the optimization target (\ref{eq:irm_target_appendix}) as a parametric constrained optimization problem with constrain:


\vspace{\baselineskip}
\begin{equation}
    \label{eq:irm_constrain_set_appendix}
    \beta^{*} \in  \underset{e \in \mathcal{E}}{\cap} \tikzmarknode{omega}{\highlight{blue}{$\Omega_e(\phi)$}}, 
\end{equation}
\begin{tikzpicture}[overlay,remember picture,>=stealth,nodes={align=left,inner ysep=1pt},<-]
    \path (omega.north) ++ (0,0.5em) node[anchor=south west,color=blue!67] (scalep){$\underset{\beta}{\arg \min } \mathcal{R}^{e}(\phi, \beta)$};
    \draw [color=blue!87](omega.north) |- ([xshift=-0.3ex,color=blue]scalep.south east);
\end{tikzpicture}

where the parametric constrained set for each environment is $\Omega_e(\phi) = \underset{\beta}{\arg \min } \mathcal{R}^{e}(\phi, \beta)$ (Note that $\Omega_e(\phi)$ can be a set with cardinality bigger than 1, since the optimal linear classifier may not be unique). The constraint (\ref{eq:irm_constrain_set_appendix}) implies that $\beta^*$ lies in the joint set of $\Omega_e(\phi)$, which also means that there is an element in each $\Omega_e(\phi)$ equal to $\beta^*$. We refer to such element to be $\beta_e \in \Omega_e(\phi)$, and we have the alternative form: 

\begin{align}
   \forall e \in \mathcal{E}, \exists \beta_e \in \Omega_e(\phi), \beta^{*} = \beta_e
   \label{eq:ele_in_env_appendix}
\end{align}

Equivalently, 


\vspace{\baselineskip}
\begin{equation}
   \forall e \in \mathcal{E}, \exists \beta_e \in \Omega_e(\phi), \beta_e \in \tikzmarknode{betastar}{\highlight{blue}{$\underset{e' \in \mathcal{E}  \backslash e}{\cap} \Omega_{e'}(\phi)$}}
   \label{eq:no_beta_star_appendix}
\end{equation}
\begin{tikzpicture}[overlay,remember picture,>=stealth,nodes={align=left,inner ysep=1pt},<-]
    \path (betastar.north) ++ (0,0.5em) node[anchor=south west,color=blue!67] (scalep){by (\ref{eq:irm_constrain_set_appendix}) and (\ref{eq:ele_in_env_appendix})};
    \draw [color=blue!87](betastar.north) |- ([xshift=-0.3ex,color=blue]scalep.south east);
\end{tikzpicture}

The interpretation of constraint (\ref{eq:no_beta_star_appendix}) is that --- for all environments, there is a hyperplane in the optimal set $\Omega_{e}(\phi)$ that also lies in the intersection of other environments' optimal set ($\underset{e' \in \mathcal{E}  \backslash e}{\cap} \Omega_{e'}(\phi)$). 
Now we rewrite the optimization target (\ref{eq:irm_target}) as:

\begin{align}
    & \min _{\phi, \beta_{e^{(1)}}, ..., \beta_{e^E}} \frac{1}{|\mathcal{E}|} \sum_{e \in \mathcal{E}} \mathcal{R}^{e}(\phi, \beta_e) \quad  \label{eq:irm_target_v1} \\ 
    &\text { s.t. } \forall e \in \mathcal{E}, \exists \beta_e \in \Omega_e(\phi), \beta_e \in   \underset{e' \in \mathcal{E} \backslash e}{\cap} \Omega_{e'}(\phi)
\label{eq:irm_constrain_v1}
\end{align}

In this way, we can get rid of finding a unique $\beta^*$, but instead optimizing multiple linear classifiers $\beta_{e^{(1)}}, ..., \beta_{e^E}$, which is easier to optimize in a relaxed form as we will show next. 

One key challenge for this optimization problem is that there is no guarantee that $\underset{e' \in \mathcal{E} \backslash e}{\cap} \Omega_{e'}(\phi)$ is non-empty for a feature extractor $\phi$ and $\beta_e$. We therefore relax the optimization target as: \vspace{-0.2cm}


\vspace{\baselineskip}
\begin{align}
    & \min _{\phi, \epsilon, \beta_{e^{(1)}}, ..., \beta_{e^E}} \frac{1}{|\mathcal{E}|} \sum_{e \in \mathcal{E}} \mathcal{R}^{e}(\phi, \beta_e) \quad  \\ 
    &\text { s.t. } \forall e \in \mathcal{E}, \exists \beta_e \in \Omega_e(\phi),
    \tikzmarknode{approx}{\highlight{blue}{$\underset{e' \in \mathcal{E} \backslash e}{\max} \|\beta_e - \Omega_{e'}(\phi)\|_2 \le \epsilon$}}, 
   \label{eq:irm_constrain_v2_appendix}
\end{align}
\begin{tikzpicture}[overlay,remember picture,>=stealth,nodes={align=left,inner ysep=1pt},<-]
    \path (approx.south) ++ (0,-0.5em) node[anchor=north east,color=blue!67] (scalep){relax $\beta_e \in \underset{e' \in \mathcal{E} \backslash e}{\cap} \Omega_{e'}(\phi)$};
    \draw [color=blue!87](approx.south) |- ([xshift=-0.3ex,color=blue]scalep.south west);
\end{tikzpicture}

where we define the $l_2$ distance between a vector $\beta$ and a set $\Omega$ as : $\|\beta - \Omega\|_2 = \underset{\upsilon \in \Omega}{\min} \|\beta - \upsilon\|_2 $.

Practically, $\epsilon$ can be set to be any variable converging to 0 during the optimization stage. Without losing the generality, we change the constraint (\ref{eq:irm_constrain_v2_appendix}) to the following form: 

\begin{align}
\begin{split}
 & \forall e \in \mathcal{E}, \exists \beta_e \in \Omega_e(\phi), \\ 
 & \underset{e' \in \mathcal{E} \backslash e}{\max}  \ \underset{\beta_{e'} \in \Omega_{e'}(\phi)}{\min}\|\beta_e - \beta_{e'}\|_2 \le \\ & \quad \quad \quad \quad \quad \quad \alpha 
 \underset{e' \in \mathcal{E} \backslash e}{\max} \ \underset{\beta_{e'} \in \Omega_{e'}(\phi)}{\min}\|\beta_e - \beta_{e'}\|_2, 
\label{eq:irm_constrain_v4_appendix}
\end{split}
\end{align}

where $\alpha \in (0,1) $. Note that constraint (\ref{eq:irm_constrain_v4_appendix}) will be satisfied only when $\underset{e' \in \mathcal{E} \backslash e}{\max}  \ \underset{\beta_{e'} \in \Omega_{e'}(\phi)}{\min}\|\beta_e - \beta_{e'}\|_2 = 0$. Therefore,  constraint (\ref{eq:irm_constrain_v4_appendix}) is equivalent to constraint (\ref{eq:no_beta_star_appendix}), and thus equivalent to the original constraint (\ref{eq:irm_constrain_appendix}). 

If we let the set 

\begin{align}
\begin{split}
    & \Upsilon_{\alpha}(\beta_e) = \{\upsilon | \underset{e' \in \mathcal{E} \backslash e}{\max}  \ \underset{\beta_{e'} \in \Omega_{e'}(\phi)}{\min}\|\upsilon - \beta_{e'}\|_2 \\
    & \le \alpha \underset{e' \in \mathcal{E} \backslash e}{\max}  \ \underset{\beta_{e'} \in \Omega_{e'}(\phi)}{\min}\|\beta_e - \beta_{e'}\|_2\}
\end{split}
\end{align}

Then the constraint (\ref{eq:irm_constrain_v4_appendix}) can be simplified to 
\begin{align}
 & \forall e \in \mathcal{E}, \exists \beta_e \in \Omega_e(\phi), \beta_e \in \Upsilon_{\alpha}(\beta_e)
\label{eq:irm_constrain_v5_appendix}
\end{align}

\end{proof}

\subsection{Projected Gradient Optimization for PG-IRM objective}
\label{sec:proof_s2}
We proceed with introducing how the Projected Gradient Descent can effectively optimize the PG-IRM objective. We start by introducing the background of  the Projected Gradient Descent algorithm. 


Projected Gradient Descent is commonly applied in constrained optimization, which aims to find a point $\theta$ achieving the smallest value of some loss function $\mathcal{L}$ subject to the requirement that $\theta$ is contained in the feasible set $\Omega$. Formally, the objective can be written as: 
$$\min _{\theta \in \Omega} \mathcal{L}(\theta)$$

If we minimize the objective $\mathcal{L}(\theta)$ by gradient descent, we have $$(\text{GD})\quad \theta := \theta - \gamma \nabla \mathcal{L}(\theta),$$
where $\gamma$ is the step size. However, it is not guaranteed that the updated $\theta$ still falls into the set $\Omega$. The projected gradient descent (PGD) algorithm is designed to project the solution back in the feasible set. Formally, 
$$(\text{PGD})\quad\theta := P_{\Omega}(\theta - \gamma \nabla \mathcal{L}(\theta)),$$
where the $P_{\Omega}(\cdot)$ is defined as the Euclidean Projection:

$$P_{\Omega}(u)=\underset{v \in \Omega}{\arg \min} \|u-v\|_2 $$

In the PG-IRM objective, we have the constraint set $\Omega = \Upsilon_{\alpha}(\beta_e)$, we show in the next Lemma~\ref{lemma:linear_interp}
 that the Euclidean Projection from $\beta_e$ to $\Upsilon_{\alpha}(\beta_e)$ is equivalent to the linear interpolation between $\beta_e$ and the farthest hyperplane $\beta_{\bar{e}}$ for environment $\bar{e}$.

\begin{lemma} Given that 

\begin{align*}
\begin{split}
    & \Upsilon_{\alpha}(\beta_e) = \{\upsilon | \underset{e' \in \mathcal{E} \backslash e}{\max}  \ \underset{\beta_{e'} \in \Omega_{e'}(\phi)}{\min}\|\upsilon - \beta_{e'}\|_2 \\
    & \le \alpha \underset{e' \in \mathcal{E} \backslash e}{\max}  \ \underset{\beta_{e'} \in \Omega_{e'}(\phi)}{\min}\|\beta_e - \beta_{e'}\|_2\}
\end{split}
\end{align*}
We have:
    $$P_{\Upsilon_{\alpha}(\beta_e)}(\beta_e) = \alpha
    \beta_e + (1 - \alpha)  \beta_{\bar{e}},$$
where $\beta_{\bar{e}}$ is selected with $\bar{e}  = \underset{e' \in \mathcal{E}  \backslash e }{\text{argmax}} \|\beta_e - \beta_{e'}\|_2$.
\label{lemma:linear_interp}
\end{lemma}

\begin{proof}
    We give the proof in an intuitive way shown in Figure~\ref{fig:proof}. Specifically, the feasible region $\Upsilon_{\alpha}(\beta_e)$ can be regarded as an intersection of several hyper-spheres centered with all domain-wise live-vs-spoof hyperplanes $\beta_{e'}$. The radius is given by the $\alpha$ multiplying the distance to the farthest hyperplane $\beta_{\bar{e}}$. Therefore the Euclidean projection of $\beta_e$ to the feasible set simultaneously lies on the surface of the hypersphere and the line segments between $\beta_e$ and $\beta_{\bar{e}}$. It can be easily verified that $$P_{\Upsilon_{\alpha}(\beta_e)}(\beta_e) = \alpha
    \beta_e + (1 - \alpha)  \beta_{\bar{e}},$$ satisfies the given criteria.

\end{proof}

\input{subtex/figure_proof.tex}

\input{subtex/alg_1}


\paragraph{Main results.} When we have the projected form on the constraint set, deriving the optimization strategy is thus straightforward. As shown in Alg.~\ref{alg:proj_grad_appendix}, we first calculate the gradient of  hyperplanes for all domains
$$\tilde{\beta}^{t+1}_e = \beta^{t}_e - \gamma \nabla_{\beta^{t}_e} \mathcal{L}_{\textit{PG-IRM}}.$$ We then select the farthest domain-wise hyperplanes $\beta_{\bar{e}}$ from other environments. The final projection results are thus given by $$\beta_e^{t+1}=\alpha^{\prime} \beta_e^{t+1}+\left(1-\alpha^{\prime}\right) \beta_{\bar{e}},$$ as we demonstrated in Lemma~\ref{lemma:linear_interp}.

\paragraph{Remark on the $T_a$.} In the first $T_a$ epochs, we let the feature encoder $\phi$ and domain-wise hyperplanes $\beta_e$ trained in a standard way. The goal is to ensure that the hyperplanes $\beta_e$ will reach or be close to the minimum of the domain-wise empirical risk, and we have:
$$\beta_e \in \Omega_e(\phi).$$ In Alg.~\ref{alg:proj_grad_appendix}, we use an additional parameter $\alpha'$ to manifest this procedure:
$$\alpha^{\prime}:=1-\mathbf{1}_{t>T_a}(1-\alpha)$$
Specifically, when $t<T_a$, $\alpha'=1$, which means the original gradient descent algorithm is applied. When $t<T_a$, $alpha'=\alpha$, the projected gradient descent takes charge.


\section{Why do we need a fair setting?}
\label{sec:whyfairsetting}

By visualizing the line plot of the HTER performance over 100 training epochs in Fig.~\ref{fig:test_curve}, we realize the test performance on the unseen domain is highly testset-dependent and unstable especially in the early epochs. Therefore, the best number reported commonly adopted in existing literature~\cite{wang2022patchnet, jia2020ssdg, wang2022ssan} usually happens in an unpredictable earlier epoch. Such ``best'' snapshot is also hard to be selected by validation strategy because we have zero information regarding the test domain. As an alternative, we noticed that the test performance is more stable in the last 10 epochs upon convergence, which motivates us to propose using a fairer comparison strategy introduced in Section~\ref{sec:exp}.

\begin{figure}[htb]
    \centering
    \includegraphics[width=0.95\linewidth]{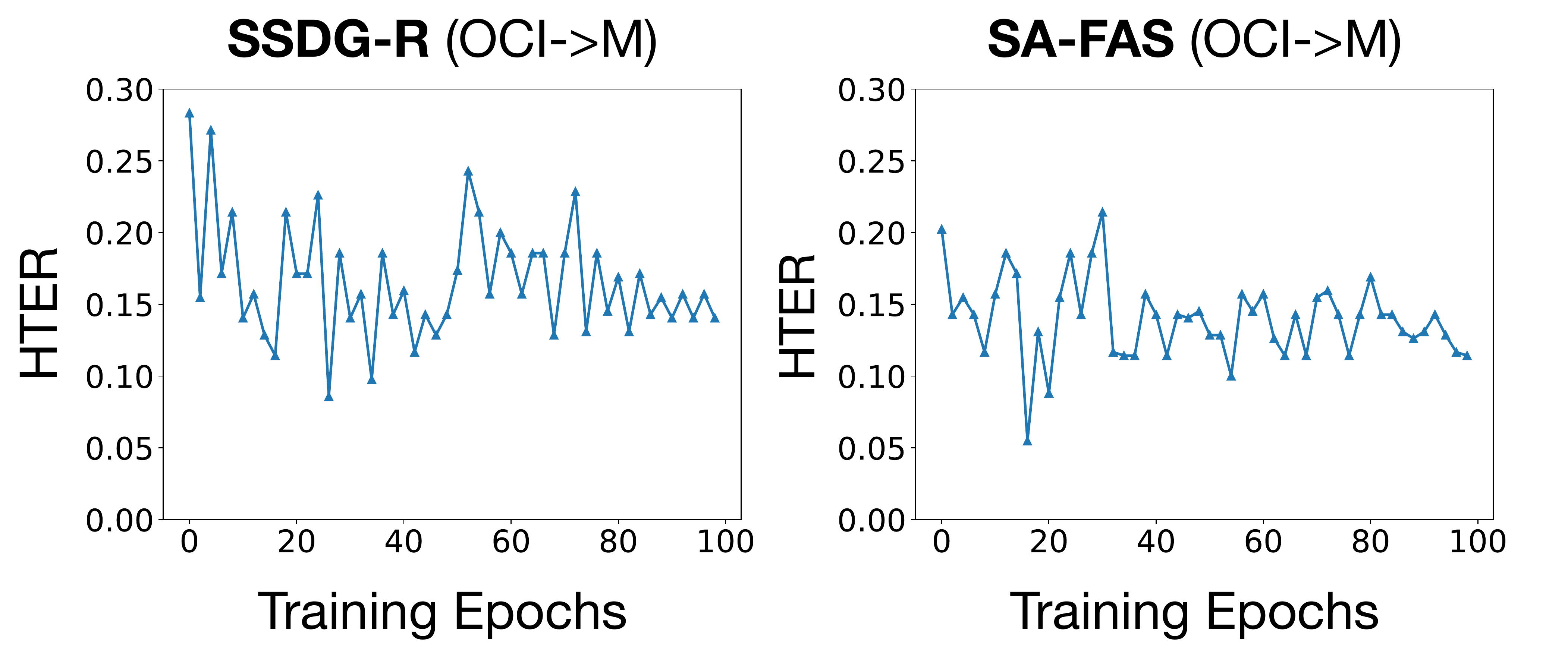}
    \caption{The line plot of the HTER performance tested on \texttt{MSU} dataset when trained on \texttt{CASIA}, \texttt{Replay} and \texttt{OULU} with SSDG-R~\cite{jia2020ssdg} and SA-FAS  over 100 training epochs. }
    \label{fig:test_curve}
\end{figure}

\section{Convergence of PG-IRM}
\label{sec:cosine_curve}

Recall that in PG-IRM, we optimize multiple linear classifiers simultaneously $\beta_{e^{(1)}}, \beta_{e^{(2)}}, \beta_{e^{(3)}}$ and gradually align them during training. In this section, we would like to verify if PG-IRM indeed regularizes domain classifiers to be close to each other and finally converges to the same one $\beta^*=\beta_{e^{(1)}}=\beta_{e^{(2)}}=\beta_{e^{(3)}}$. Empirically, we use the averaged cosine distance between domain classifiers to measure the distance between them: 

$$S_{\text{cos}}= \mathbb{E}_{e,e'\in \mathcal{E}, e \neq e'}[cos(\beta_e,\beta_{e'})]$$

As shown in Fig.~\ref{fig:consine_curve}, the averaged cosine value between domain classifiers diminishes gradually and finally converges to 1, which suggests that they converge to a $\beta^*$ that is aligned for all domains. 

\begin{figure}[htb]
    \centering
    \includegraphics[width=0.8\linewidth]{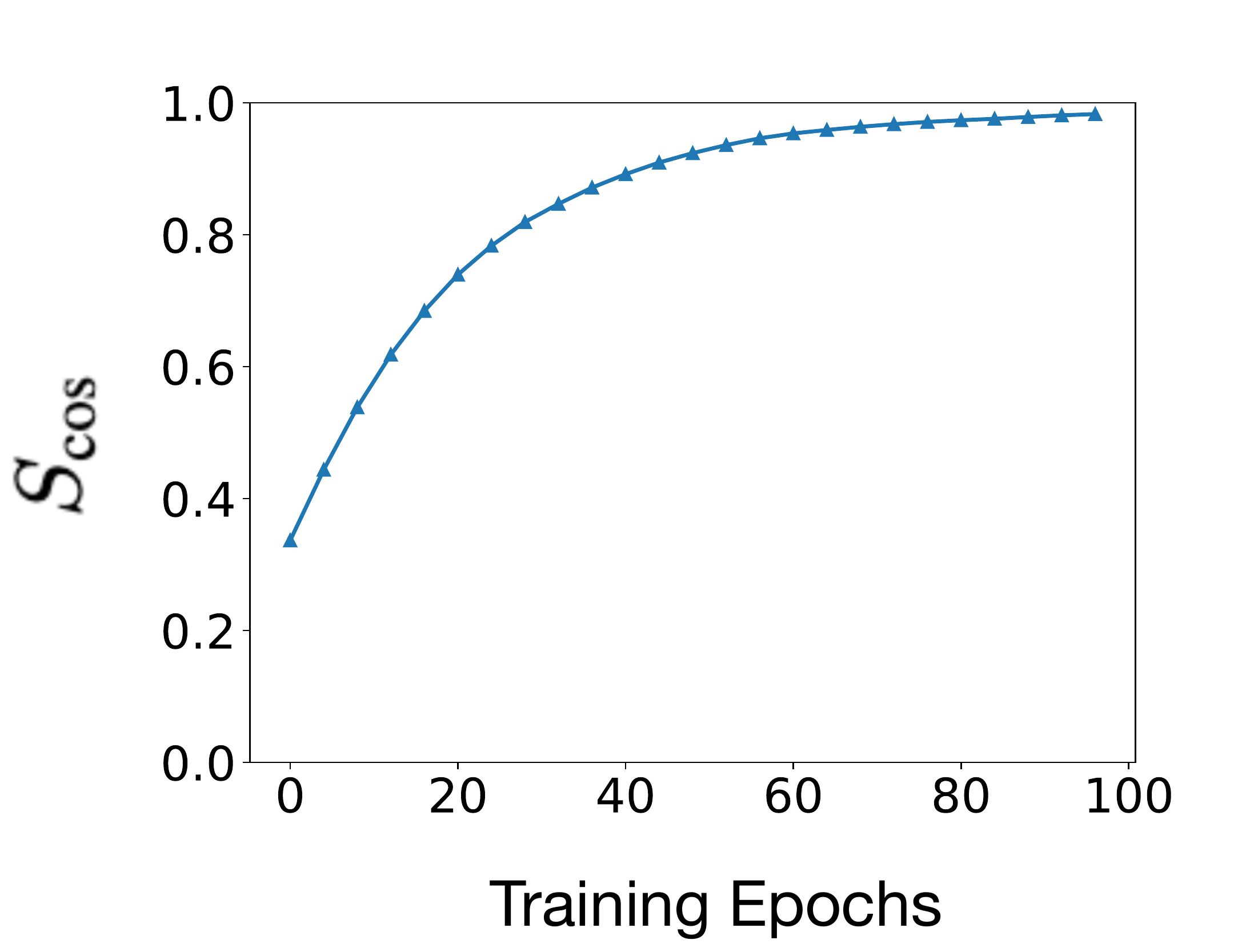}
    \caption{The line plot of the $S_{\text{cos}}$ when trained on \texttt{CASIA}, \texttt{Replay} and \texttt{OULU} with PG-IRM over 100 training epochs. }
    \label{fig:consine_curve}
\end{figure}

\section{Sensitivity Analysis}
\label{sec:sensitivity}

In this section, we perform the sensitivity analysis of hyper-parameter settings for SA-FAS in Fig.~\ref{fig:hyper}. The performance comparison in the bar plot for each hyper-parameter is reported by fixing other hyper-parameters. In the figure, we observe that the performance of SA-FAS is less sensitive to the learning rate and the alignment starting epoch compared with the maximum gap of $1.2\%$ in the given range. We also notice that choosing the right alignment parameter $\alpha$ is more important, since a proper $\alpha$ ensures the domain-wise decision boundaries are aligned not too fast and not too slow. In the extreme case, if $\alpha=0$, it degenerates to the ERM after epoch $T_a$ and if $\alpha=1$, the domain-wise boundaries will never get aligned with each other. In summary, our algorithm does not require heavy hyper-parameter tuning as long as it falls into a reasonable range.

\begin{figure*}[htb]
    \centering
    \includegraphics[width=1\linewidth]{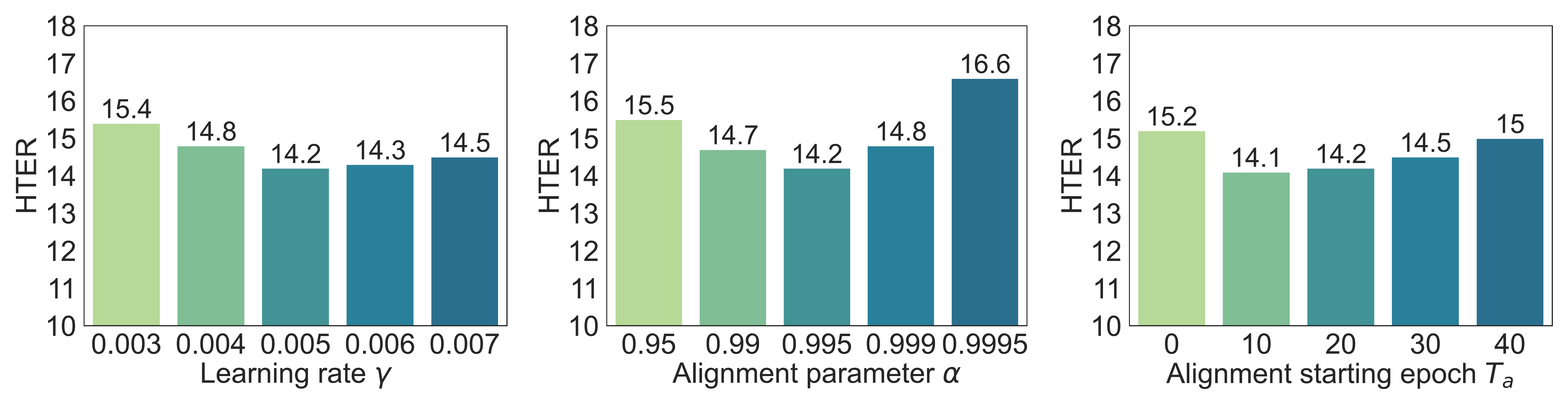}
    \caption{Sensitivity analysis of hyper-parameters: learning rate $\gamma$, alignment parameter $\alpha$, alignment starting epoch $T_a$. The HTER is reported on the mean performance based on the last 10 epochs.  The middle bar in each plot corresponds to the hyperparameter value used in our main experiments.  }
    \label{fig:hyper}
\end{figure*}

\section{Limitation}
\label{sec:limit}
Our work has two limitations. Firstly, our framework assumes the dataset collected from each domain contains both live and spoof data. For example, SA-FAS can not handle the  training data  with  live samples only from domain A and spoof samples only from domain B. Secondly, SA-FAS may cause extra computation costs when the domain amount is very large since we set up one hyperplane for each domain.

\input{subtex/figure_corr_supp}

%% file: subtex/figure_method_sup.tex
\begin{figure*}[htb]
    \small\centering
    \includegraphics[width=0.95\textwidth]{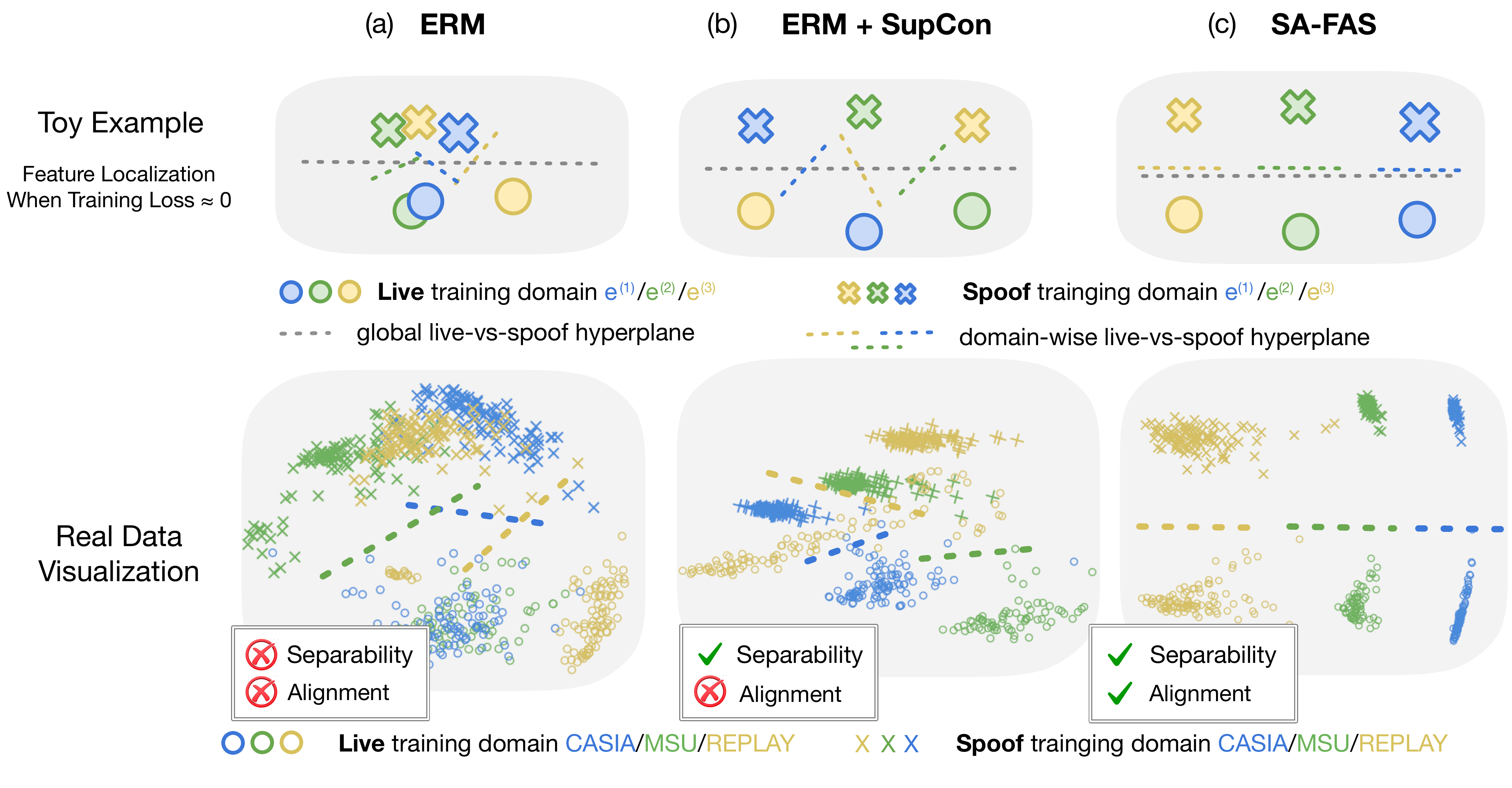}
    \caption{\small Illustration of feature space with three optimization objectives (ERM/ERM+SupCon/SA-FAS). For each objective, the first row shows the feasible solution in toy examples where each domain with a live/spoof label is represented by one circle/cross. The second role shows the visualization of real data via linear projection. The visualization is conducted by inserting and scattering the features from a 2-dimensional hidden layer between the penultimate layer and the final output layer. }%
    \label{fig:method-sup}
\end{figure*}


%% file: subtex/figure_proof.tex
\begin{figure}[htb]
    \small\centering
    \includegraphics[width=.45\textwidth]{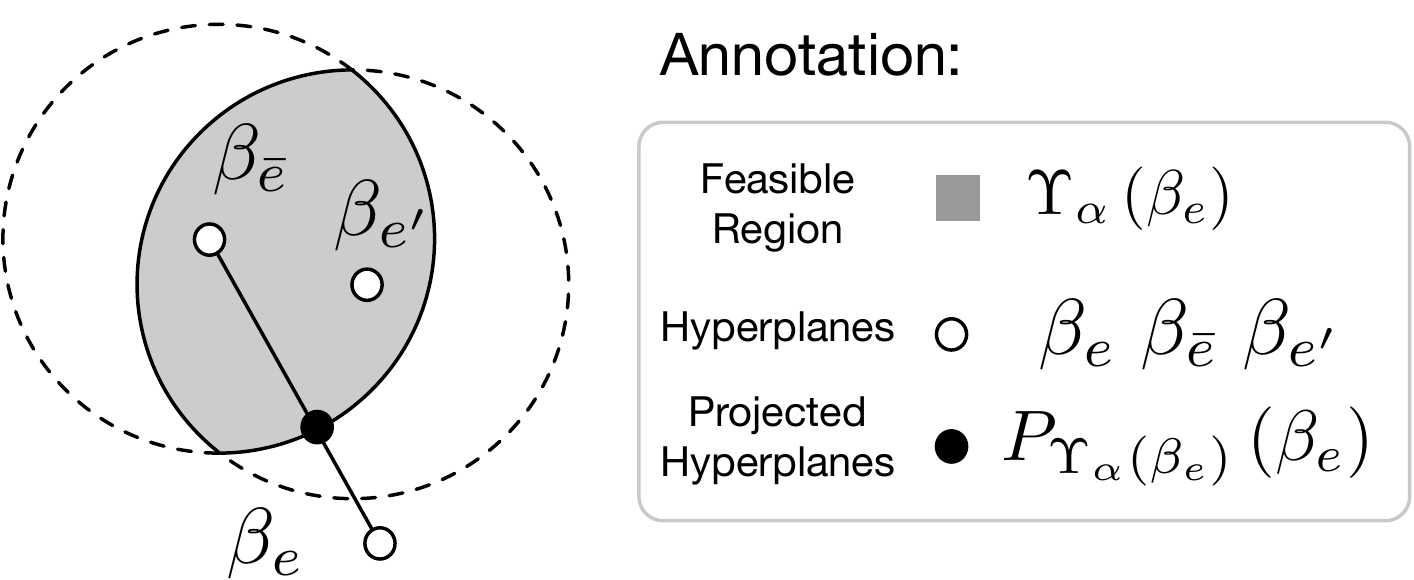}
    \caption{\small Illustration of the Euclidean projection results (solid black dot) to the feasible set $\Upsilon_\alpha\left(\beta_e\right)$. }
    \figvspace
    \label{fig:proof}
\end{figure}

%% file: subtex/alg_1.tex
\begin{algorithm}[t]
\begin{algorithmic}\State Initialize $\phi, \beta_{e^{(1)}}, ... , \beta_{e^{(E)}}$, learning rate $\gamma$, alignment parameter $\alpha$, alignment starting epoch $T_a$.
\For{$\text{t in 0, 1, ..., }$}
\State Run forward pass and calculate the gradient.
\For{$e \in \mathcal{E}$}
    
    \State $\tilde{\beta}^{t+1}_e = \beta^{t}_e - \gamma \nabla_{\beta^{t}_e} \mathcal{L}_{\textit{PG-IRM}}$
    \State $\alpha' := 1 - \mathbf{1}_{t > T_a} (1 - \alpha)$ 
    \State select $\beta^{t}_{\bar{e}}$ with $\bar{e}  = \underset{e' \in \mathcal{E}  \backslash e }{\text{argmax}} \|\tilde{\beta}^{t+1}_e - \beta^{t}_{e'}\|_2$
    \State $\beta^{t+1}_e = \alpha'
    \tilde{\beta}^{t+1}_e + (1 - \alpha')  \beta^t_{\bar{e}}$
\EndFor
\State Update $\phi^{t+1} = \phi^{t} - \gamma \nabla_{\phi^t} \mathcal{L}_{\textit{PG-IRM}}$.
\EndFor
\end{algorithmic}
\caption{\small PG-IRM}
\label{alg:proj_grad_appendix}
\end{algorithm}

%% file: subtex/figure_corr_supp.tex
\begin{figure*}[t]
    \small\centering
    \includegraphics[width=\textwidth]{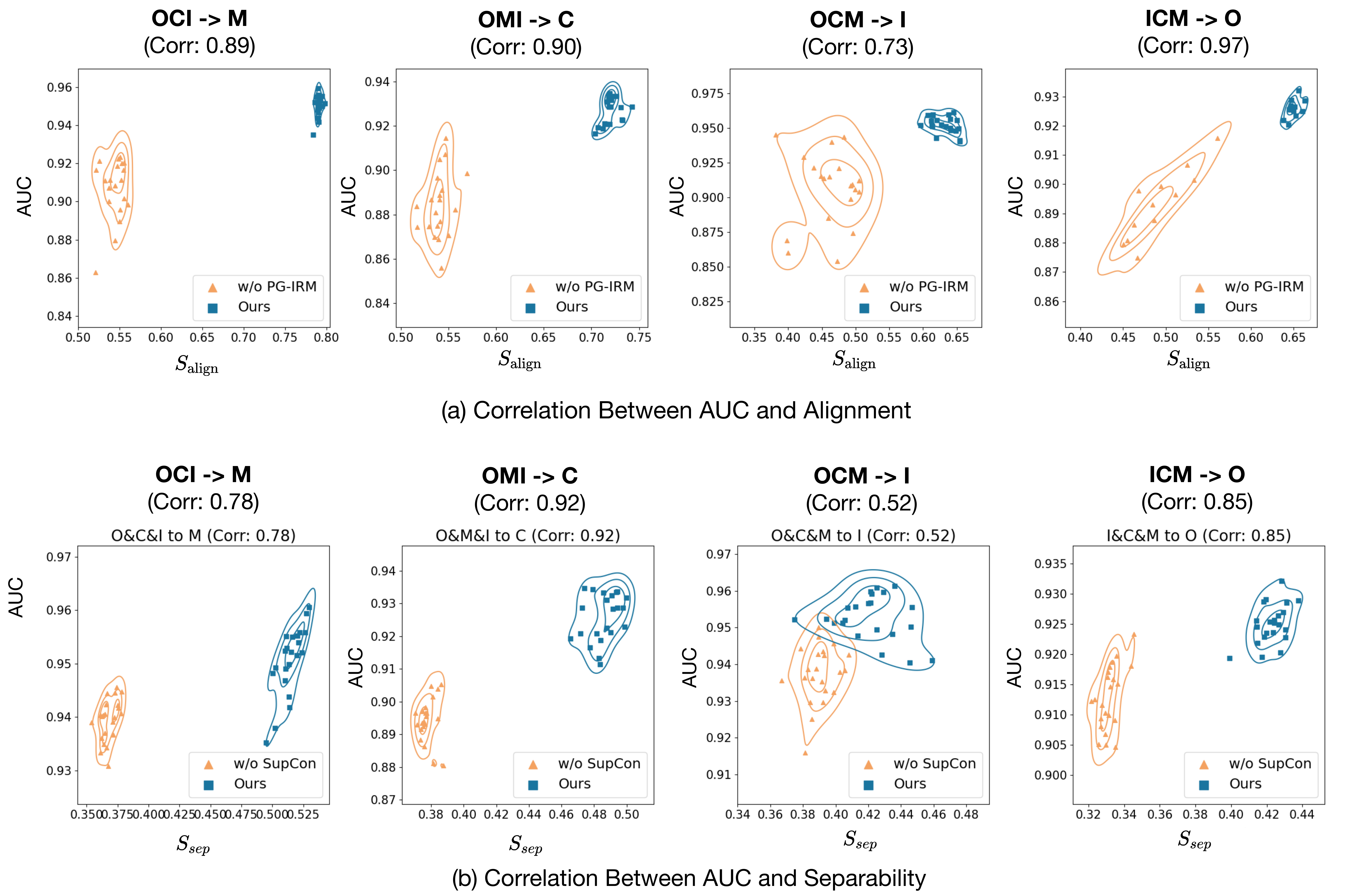}
    \caption{\small Correlation between the test performance AUC and two
properties measure. Each dot represents one snap-shot during the training stage in four cross-domain settings.}
\label{fig:more_corr}
\end{figure*}